\documentclass{article}

\usepackage[preprint]{neurips_2021}

\usepackage{comment}
\usepackage[utf8]{inputenc} 
\usepackage[T1]{fontenc}    
\usepackage{hyperref}       
\usepackage{url}            
\usepackage{booktabs}       
\usepackage{amsfonts}       
\usepackage{nicefrac}       
\usepackage{microtype}      
\usepackage{xcolor}         
\usepackage{newfloat}
\usepackage{soul}
\usepackage{graphicx}
\usepackage{amsmath}
\usepackage{amsfonts}
\usepackage{amsthm}
\usepackage[boxruled]{algorithm2e}
\usepackage[framemethod=tikz]{mdframed}
\usepackage{url}
\usepackage{tikz}
\usepackage{pgfplots}
\usetikzlibrary{arrows,positioning}
\usepackage[tikz]{bclogo}
\usepackage{lipsum}
\usepackage[many]{tcolorbox}
\usepackage{mdframed}
\usepackage{setspace}

\SetCommentSty{mycommfont}

\definecolor{bgblue}{RGB}{245,243,253}
\definecolor{ttblue}{RGB}{91,194,224}
\DeclareFloatingEnvironment[listname={List of Insets},name={Inset}]{inset}
\DeclareMathOperator{\Var}{Var}

\title{Algorithmic Bias and Data Bias: Understanding the Relation between Distributionally Robust Optimization and Data Curation}

\author{%
  Agnieszka S{\l}owik \\
  Department of Computer Science and Technology\\
  University of Cambridge\\
  Cambridge, UK \\
  \texttt{agnieszka.slowik@cl.cam.ac.uk} \\
  \And
   L\'{e}on Bottou \\
   Facebook AI Research, New York, NY, USA,\\
and New York University, New York, NY, USA
}

\theoremstyle{plain}
\newtheorem{theorem}{Theorem}
\newtheorem{lemma}{Lemma}
\theoremstyle{definition}
\newtheorem{definition}{Definition}
\theoremstyle{remark}
\newtheorem{example}{Example}

\newtheorem*{thexrefthm}{{\bf Theorem \thexref}}
\newenvironment{xreftheorem}[1]{%
\def\thexref{\ref{#1}}\begin{thexrefthm} \em}{%
\end{thexrefthm}}
\newtheorem*{thexrefdef}{{\bf Definition \thexref}}

\def\eg.{\mbox{e.}\mbox{g.}}
\def\R{\ensuremath{\mathbb{R}}}
\def\dotp#1#2{\ensuremath{\left<{#1},{#2}\right>}}
\def\Q{\ensuremath{\mathcal{Q}}}
\def\E{\ensuremath{\mathbb{E}}}
\def\Pmix{{P_{\mathrm{mix}}}}
\def\Pmixeps{{P_{\mathrm{mix}}^{(\varepsilon)}}}
\def\Pmixoot{{P_{\mathrm{mix}}^{(1/t)}}}

\def\Qmix{{\Q_{\mathrm{mix}}}}

\pgfplotsset{compat=1.17}

\makeatletter

\belowdisplayskip \abovedisplayskip
\def\section{\@startsection{section}{1}{\z@}{-1.6ex \@plus -0.2ex \@minus -0.1ex}{1ex \@plus 0ex \@minus 0.2ex}{\large\bf\raggedright}}
\makeatother

\begin{document}

\maketitle

\begin{abstract}
Machine learning systems based on minimizing average error have been shown to perform inconsistently across notable subsets of the data, which is not exposed by a low average error for the entire dataset. In consequential social and economic applications, where data represent people, this can lead to discrimination of underrepresented gender and ethnic groups. Given the importance of bias mitigation in machine learning, the topic leads to contentious debates on how to ensure fairness in practice (data bias versus algorithmic bias). Distributionally Robust Optimization (DRO) seemingly addresses this problem by minimizing the worst expected risk across subpopulations. We establish theoretical results that clarify the relation between DRO and the optimization of the same loss averaged on an adequately weighted training dataset. The results cover finite and infinite number of training distributions, as well as convex and non-convex loss functions. We show that neither DRO nor curating the training set should be construed as a complete solution for bias mitigation: in the same way that there is no universally robust training set, there is no universal way to setup a DRO problem and ensure a socially acceptable set of results. We then leverage these insights to provide a mininal set of practical recommendations for addressing bias with DRO. Finally, we discuss ramifications of our results in other related applications of DRO, using an example of adversarial robustness. Our results show that there is merit to both the algorithm-focused and the data-focused side of the bias debate, as long as arguments in favor of these positions are precisely qualified and backed by relevant mathematics known today.

\end{abstract}

\section{Introduction}

Machine learning algorithms are increasingly used to support real-world decision-making processes. In that context, optimising for the loss averaged on the overall population can easily
yield models that perform poorly on specific subpopulations, potentially amplifying the injustices that already plague our society \citep{DBLP:journals/corr/DattaTD14, DBLP:journals/bigdata/Chouldechova17,amazon, DBLP:conf/nips/RahmattalabiVFR19, NEURIPS2020_79a3308b,  nyt, propublica}. 

Whether such problems can be addressed by curating the training data leads to contentious debates. For instance, it is difficult and often painful to know whether it is sufficient to ensure that the relevant subpopulations are well represented in the training set, whether the structure of the statistical model must be revisited, or whether the whole system, its goals, and its methods, are fundamentally broken. We argue that there is merit to both the algorithm-focused and the data-focused side of the discussion, as long as arguments in favor of these positions are precisely qualified and backed by mathematics.

Distributionally Robust Optimization (DRO)~\citep{DBLP:books/degruyter/Ben-TalGN09} bridges two perspectives on this problem. On one hand, DRO seems to offer a promising solution because it minimizes the worst loss observed on multiple distributions such as those representing each subpopulation. On the other hand, it can be shown under weak conditions that DRO is closely related to minimizing the average loss on an adequate mixture of those distributions, that is, a training set in which the subpopulations have been adequately weighted. Our contributions are:
\begin{enumerate}
    \item We establish results that clarify the relation between DRO and the optimization of the same loss averaged on an adequately weighted training set (see Section~\ref{sec:mixture} summarizing the theoretical Appendix). 
    \item We also show that neither DRO nor curating the training set should be construed as a complete solution of our initial problem. In particular, each DRO formulation implicitly makes calibration assumptions on the losses measured on various subpopulations. Making them explicit brings back the contentious issues (see Section~\ref{sec:calibration}.)
    \item We leverage this mathematical understanding to provide a minimal set of practical recommendations to approach such difficult problems. Although what is acceptable or not in real-world applications is not a part of the mathematical problem, elementary mathematics tells us that we cannot obtain an acceptable result with DRO if we are unable to obtain an acceptable result with systems specifically optimized for each subpopulation. 
    \item Finally, we discuss what these insights tell us about the caveats of using DRO to tame adversarial examples \citep{szegedy-2014, kuhn2019wasserstein,augustin2020adversarial, derman2020distributional}.
\end{enumerate}

\begin{figure} 
\begin{tabular}{cc}
\includegraphics[width=.49\linewidth]{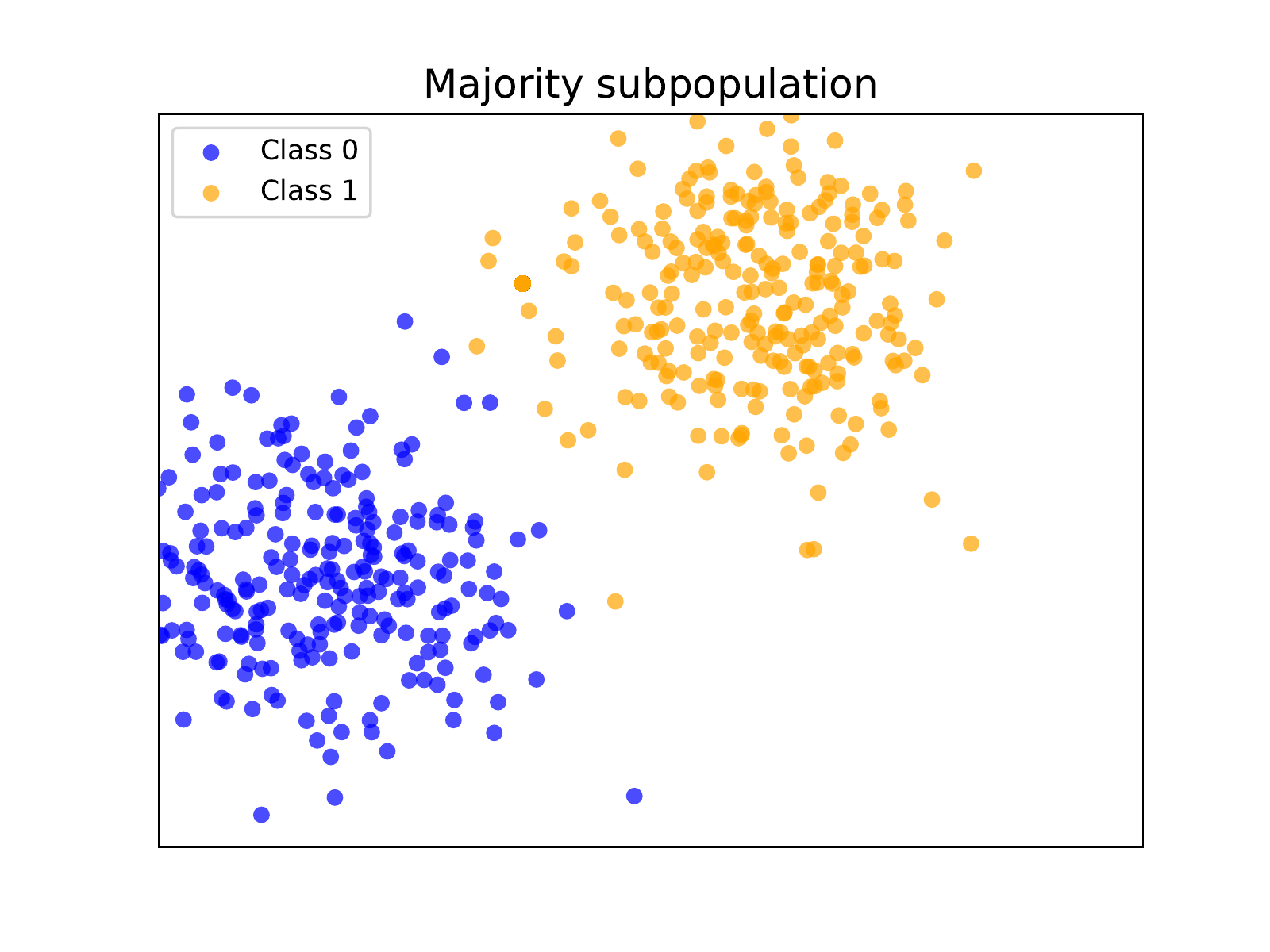} &
\includegraphics[width=.49\linewidth]{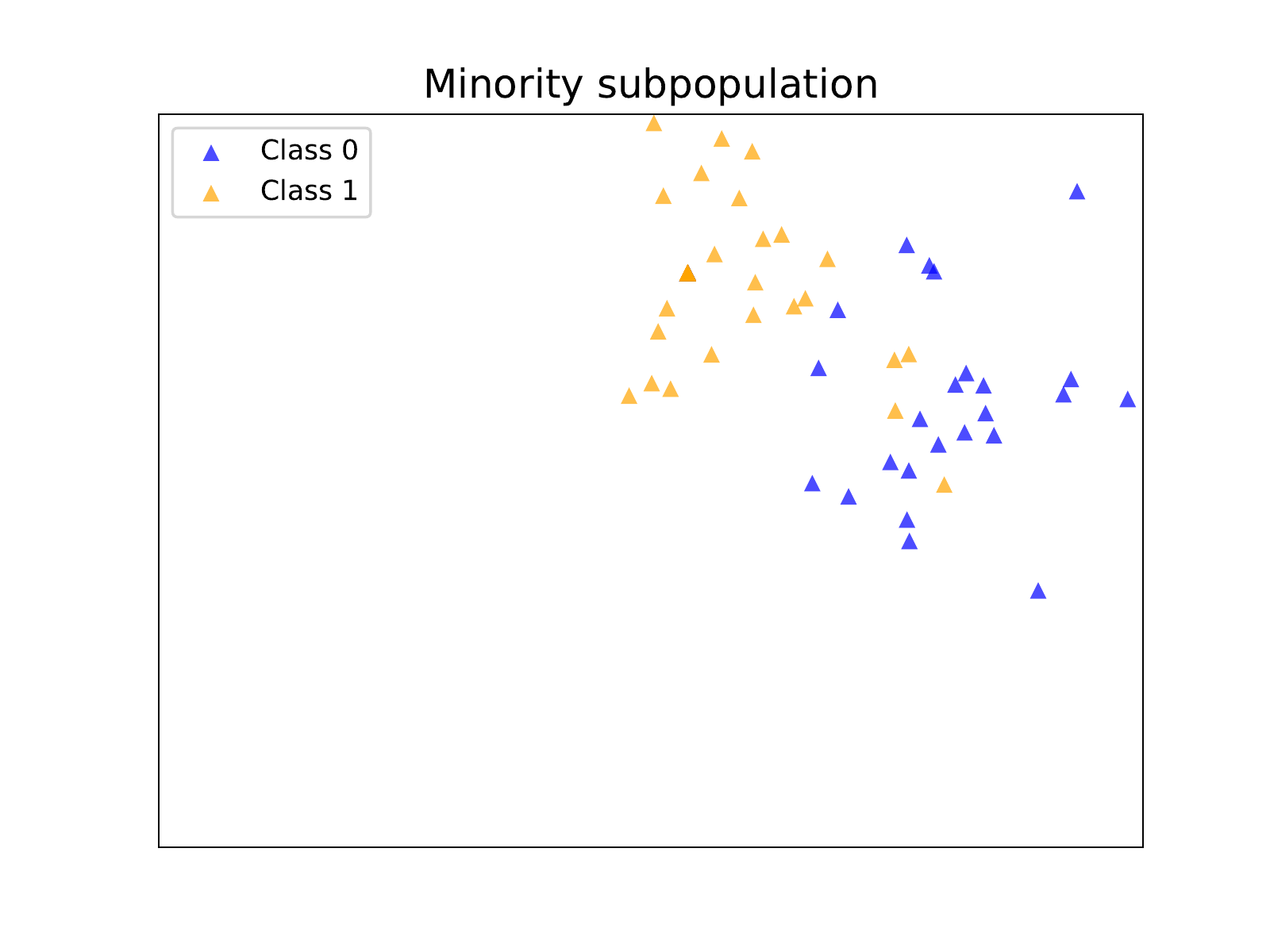} \\[-2.5ex]
\includegraphics[width=.49\linewidth]{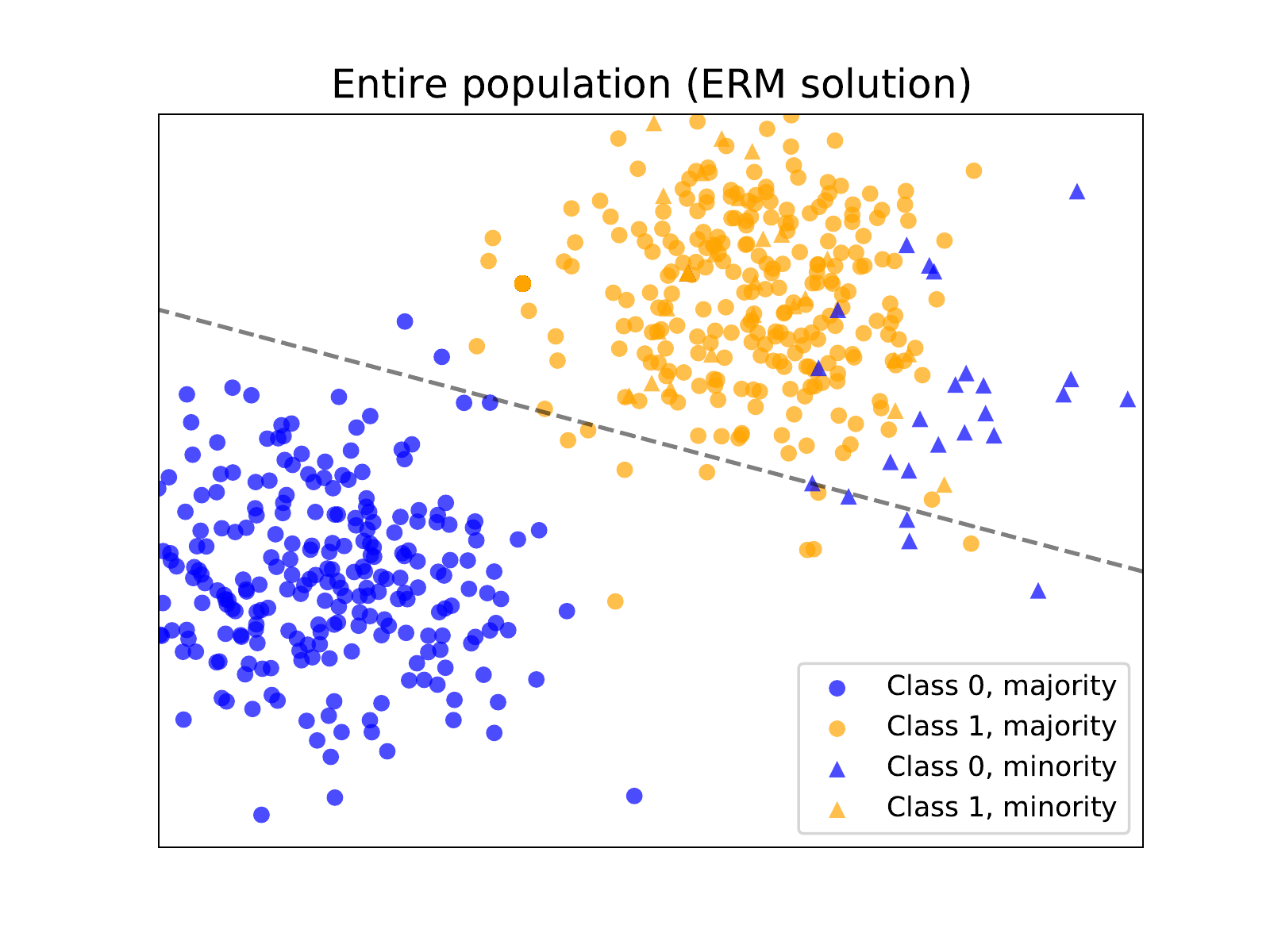} &
\includegraphics[width=.49\linewidth]{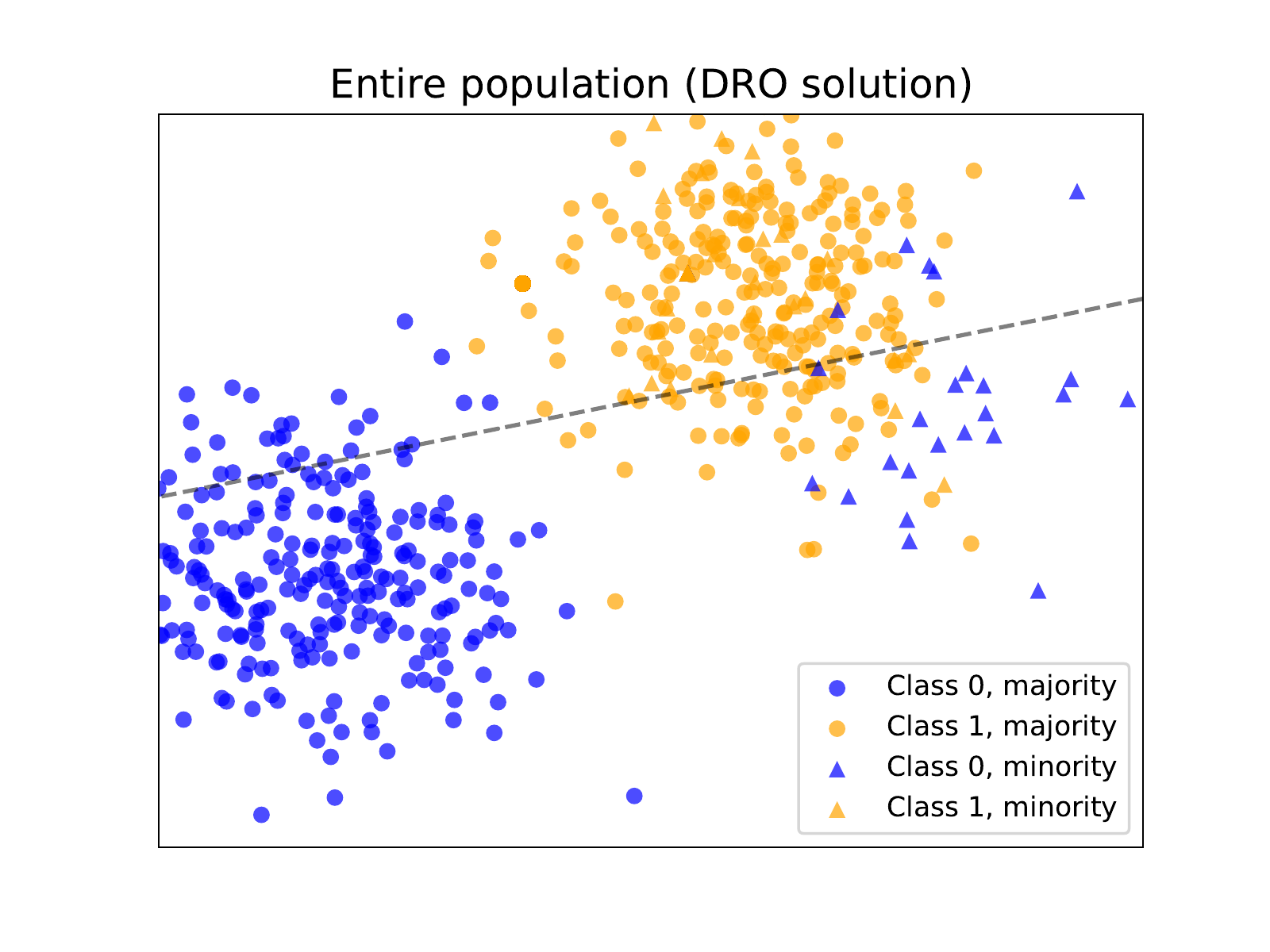} \\[-2.7ex]
\textsf{\small Majority: 97\% correct, Minority: 58\% correct} &
\textsf{\small Majority: 83\% correct, Minority: 82\% correct}\\  
\end{tabular}

\caption{Illustration of the failings of Expected Risk Minimization (ERM) and i.i.d assumption in a linear binary classification problem, where the dataset can be partitioned to a \emph{majority} and a \emph{minority} subpopulation. ERM misclassifies a significant proportion of the minority subpopulation. DRO equalizes the performance across subpopulations. The DRO solution was obtained using the Lagrangian algorithm (Algorithm~\ref{alg:lagrangianalgomain}), both solutions use the same linear SVM model.}
\label{fig:ermdro}
\vspace{-0.5cm}
\end{figure}

\section{DRO versus data curation}
\label{sec:mixture}
Traditionally, training a model in machine learning seeks parameters, say, the weights $w$ of a neural networks, that minimize a risk $C_P(w)$ that is the expectation of a loss function with respect to a \emph{single distribution of training examples}.

Alas, even when the training distribution is representative of the actual testing conditions, the trained system might perform very poorly on selected subsets of examples. For instance, Figure~\ref{fig:ermdro} describes a training problem where a majority population and a minority population have different classification boundaries. Minimizing the expected loss over the full dataset (bottom left plot) yields a system whose performance is skewed towards the majority population at the expense of the performance in the minority population. In real life, this can be a source of major injustice. Algorithms that optimize for minimum average error yield models that perform poorly on subpopulations that are already at risk due to pre-existing biases. This is most pronounced when ERM (based on minimising the average prediction error) produces solutions that privilege the majority populations over the minorities. This was shown to be consequential in scenarios such as court verdicts, loan applications, hiring and healthcare interventions~\citep{DBLP:journals/corr/DattaTD14, DBLP:journals/bigdata/Chouldechova17,amazon, DBLP:conf/nips/RahmattalabiVFR19, NEURIPS2020_79a3308b, nyt, propublica}. For example, hiring systems and ad-targeting algorithms based on minimising average error were found to discriminate against female users by more frequently proposing executive and technical jobs to men~\citep{DBLP:journals/corr/DattaTD14, amazon}.

Distributionally Robust Optimization (DRO) \emph{seemingly addresses this problem} by considering instead a collection $\Q$ of `training distributions' and minimizing the expected risk observed on the \emph{most adverse} distribution
\begin{equation}
\label{eqmain:dro}
        \min_w ~ \max_{P\in\Q}  C_P(w) ~.
\end{equation}
For instance, we can substantially reduce the performance differences reported in Figure~\ref{fig:ermdro} (bottom left plot). Using DRO with a set $\Q$ of two distributions, representing the majority and minority populations, leads to the decision boundary illustrated in Figure~\ref{fig:ermdro} (bottom right plot). This classifier does not reproduce the pre-existing imbalance between the majority and the minority subpopulation. Instead, the performance is equalized, and the average accuracy increases from $77.5\%$ to $82.5\%$.

As discussed more precisely in Section~\ref{sec:calibration}, the approach of using the basic definition of DRO \eqref{eqmain:dro} without additional information on the subpopulations is insufficient in practical applications. However, we argue that DRO remains an interesting building block because it provides a bridge between two common approaches to this problem, namely, ($i$) ensuring that the trained system has consistent performance across subpopulations, and ($ii$) curating the training set by remixing the populations until obtaining a more palatable result. 

One of the contributions of our work is an ensemble of mathematical results that clarify the relation between finding a local minimum of the DRO problem \eqref{eqmain:dro} and \eqref{eqmain:drocal} on one hand, and minimizing the usual expected risk with respect to a single, well crafted, training distribution. We hope these results will be useful for both the data-focused and the algorithm-focused sides of the bias debate in machine learning community.

For convex cost functions, these results are well known, because one can reformulate the DRO problem as a constrained optimization problem by introducing a slack variable~$L$,
    \[
       \min_{w,L} L ~~~\text{s.t.}~~~  
         \forall P\in\Q ~  C_P(w) \leq L~,
    \]
and relying on standard convex duality results~\citep{bertsekas-2009} (see Section A.4. in the Appendix).

The point of our theoretical contribution is that similar results \emph{hold for the local minima of the nonconvex costs} typical of modern deep learning systems, and also \emph{hold when the family $\Q$ is infinite}. We now summarize the main results (the elaboration and proofs of these propositions are in the Appendix).

Let $\ell(z,w)$ be the loss of a machine learning model where $w\in\R^d$ represent the parameters of the model and $z\in\R^n$ belongs to the space of examples. For instance, in least square regression, the examples $z$ are pairs $(x,y)$ and the loss is $\ell(z,w)=\|y-f_w(x)\|^2$. When the collection $\Q$ is finite, under weak regularity assumptions, a DRO local minimum is always a stationary point of the expectation of the loss with respect to a suitable mixture of the DRO training distributions. Theorem~\ref{th:one} states this fact for a finite collection $\Q$ of distributions: 
\begin{theorem}[Finite case]
\label{th:one}
Let $\Q=\{P_1,\dots,P_K\}$ be a finite set of probability distributions on~$\R^n$ and let $w^*$ be a local minimum of the DRO problem \eqref{eqmain:dro} or the calibrated DRO problem \eqref{eqmain:drocal}. Let the costs $C_P(w)=\E_{z\sim P}[\ell(z,w)]$ be differentiable in $w^*$ for all $P\in\Q$. Then there exists a mixture distribution $\Pmix=\sum_{k}\lambda_k P_k$ such that $\nabla{C}_{\Pmix}(w^*)=0$.
\end{theorem} 

When the collection $\Q$ is infinite (possibly uncountably infinite) but satisfies a \emph{tightness} condition (see Definition 1 in the Appendix), we can still show that a DRO local minimum is a stationary point for a well crafted training distribution. However, this training distribution is not necessarily a finite or countable mixture of the distributions found in $\Q$ but is always found in the weak closure of the convex hull of $\Q$. \emph{Adversarial robustness} is an example of applying DRO on an infinite family of distributions (see a discussion of the implications of Theorem~\ref{th:two} in Section~\ref{sec:adversarial}). 
    
\begin{theorem}[Infinite case]
\label{th:two}
Let $\Q$ be a tight family of probability distributions on $\R^n$. Let $w^*$ be a local minimum of problem~\eqref{eqmain:drocal}. Let $\Qmix$ be the weak convergence closure of the convex hull of $\Q$. Let there be a bounded continuous function $h(z,w)$ defined on a neighborhood $\mathcal{V}$ of $w^*$ such that $\nabla{C}_P(w)=\E_{z\sim{P}}[h(z,w)]$ for all $P\in\Qmix$ and such that $\|h(z,w)-h(z,w')\|\leq{M}\|w-w'\|$ for almost all $z\in\R^n$. Then $\Qmix$ contains a distribution $\Pmix$ such that $\nabla_w C_\Pmix(w^*)=0$.
\end{theorem}

Conversely, we consider a local minimum of the expectation of the loss with respect to an arbitrary mixture of distributions from $\Q$. Such a local minimum always is a local minimum of a \emph{calibrated DRO} problem where one introduces suitable calibration constants $r_P$ that control how we compare the costs for different distributions:
\begin{equation}
\label{eqmain:drocal}
        \min_w ~ \max_{P\in\Q} ~ \left( C_P(w) - r_P \right)~.
\end{equation}
    We shall see in Section~\ref{sec:calibration} that such calibration coefficients are in fact needed to express the subtleties of the original problem of learning from multiple subpopulations with DRO.
   
\begin{theorem}[Converse]
\label{th:converse}
Let $\Pmix=\sum_k\lambda_kP_k$ be an
arbitrary mixture of distributions $P_k\in\Q$. If $w^*$ is a local minimum of $C_\Pmix$, then $w^*$ is a local minimum of the calibrated DRO problem~\eqref{eqmain:drocal} with calibration coefficients $r_P=C_P(w^*)$.
\end{theorem}

Note that there is a discrepancy between these two theorems. Theorem~\ref{th:converse} says that a local minimum of an expected risk mixture is a DRO local minimum, but Theorem~\ref{th:one} only says that a DRO local minimum is a stationary point (that is, a point with null derivative) of an expected risk mixture. 

This distinction is moot when $\ell(z,w)$ are convex in $w$ because all stationary points are not only local minima, but also global minimum. When this is the case, Theorem~\ref{th:one} and Theorem~\ref{th:converse} then describe the same equivalence as the standard convex duality results. In contrast, when the cost functions are nonconvex, the  stationary point described in Theorem~\ref{th:one} need not even be a local minimum of the expected cost mixture. For instance, if the DRO local minimum is achieved in a region where all $C_P(w)$ have a negative curvature, then any mixture of these costs also has a negative curvature, and, as a result, the stationary point can only be a local maximum. Figure~\ref{fig:counterexamplemain} shows how such a situation can arise in theory. 

However, the learning algorithms that are typically used to train overparametrized deep learning models empirically follow trajectories where the Hessian is very flat apart from a few positive eigenvalues~\citep{sagun-2018}. Weak negative curvature directions always exist ---even when the algorithm stops making progress--- but are very weak.  In order to understand whether situations like Figure~\ref{fig:counterexamplemain} happen in practice, it makes sense to now take a closer look at the algorithms commonly used to implement DRO.

\begin{figure}[b]
    \centering
    \vspace{4ex}
    \begin{tikzpicture}
        \begin{axis}[width=8cm,height=5cm,
          axis lines=middle, 
          xmin = -3, xmax = 3, ymin = -0.6, ymax = 1.6, 
          xlabel={$w$}, ticks=none]
            \addplot[thick,blue,domain=-3:3,dashed] {
                tanh(1 + x) + (x^2 / 20)};
            \addplot[thick,green,domain=-3:3,dashed] {
                tanh(1 - x) + (x^2 / 20)};
            \addplot[thick,brown,domain=-3:3] {
                0.05 + max(tanh(1 - x) + (x^2 / 20),
                           tanh(1 + x) + (x^2 / 20) ) };  
            \addplot[thick,red,domain=-3:3] {
                (tanh(1 + x) + tanh(1 - x)) / 2 + (x^2 / 20) };  
            \node[blue] at (axis cs:-1,-0.5) {$C_1$};
            \node[green] at (axis cs:1,-0.5) {$C_2$};
            \node[brown] at (axis cs:1,1.4){
                $\max(C_1,C_2)$};
            \node[red] at (axis cs:-2,0.6){
                $\tfrac12{C_1}+\tfrac12{C_2}$};
        \end{axis}
    \end{tikzpicture}
    \caption{ \label{fig:counterexamplemain}
    Consider the two real functions $C_1(w)=\tanh(1+w)+\epsilon w^2$ and $C_2(w)=\tanh(1+w)+\epsilon w^2$.
    The minimum $w^*=0$ of $\max\{C_1(w),C_2(w)\}$ is a stationary point of the mixture cost $C_{\mathrm{mix}}(w) = \tfrac12C_1(w)+\tfrac12C_2(2)$. However, because
    it is achieved in negative curvature regions of $C_1$ and $C_2$, this stationary point is not a local minimum but a local maximum of the mixture cost. In practice, overparametrized deep learning models are unlikely to meet such a situation because the Hessian along the learning trajectories tend to be essentially flat apart from a few positive eigenvalues~\citep{sagun-2018}.}
\end{figure}
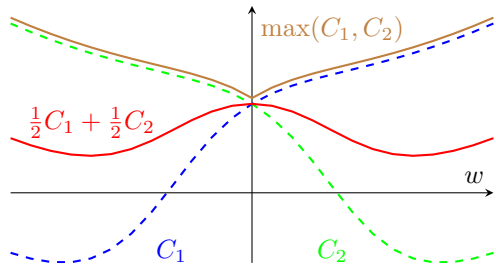

In the convex case, it is known that increasing the weight of a distribution in the mixture is equivalent to reducing the corresponding calibration coefficient. This observation leads to a plethora of saddle-point seeking algorithms such as Uzawa iterations~\citep{uzawa-58}. See Algorithm~\ref{alg:lagrangianalgomain} for a representative example. Because such algorithms are reliable and can be made efficient, many authors advocate using similar strategies for nonconvex deep learning systems (e.g., \citep{sagawa-2020}). 
Section A.4. in the Appendix shows how
our theoretical results offer support for this practice. It also shows that such an algorithm fails to find DRO minima when the associated stationary point of the expected cost mixture is not itself a local minimum. For practical purposes, this means that \emph{there is no substantial difference between using such an efficient DRO algorithm and minimizing a well-crafted expected risk mixture.} 

\IncMargin{1em} 
\begin{algorithm}[tb]
\label{alg:lagrangianalgomain}
\DontPrintSemicolon
\LinesNumbered
\SetKw{KwInput}{Input:}
\SetKw{KwOutput}{Output:}
\SetKwFunction{Descend}{Descend}
\SetKwFunction{Cost}{Cost}
\SetKwFunction{Acceptable}{Acceptable}
\SetKwFunction{ArgMax}{ArgMax}
\KwInput{Equally sized training sets $D_k$ for $k=1\dots K$}\;
\KwInput{Calibration coefficients $r_k$. 
         Initial weights $w_0$.}\;
\KwInput{Temperature $\beta$. 
         Stopping threshold $\epsilon$.}\;
\KwOutput{A sequence of weights $w_t$.}\;
\BlankLine
$t \leftarrow K$\;
$\lambda_k \leftarrow 1/t ~~\forall k$\;
\Repeat {$\max_k|\lambda_k-\delta_k| < t\epsilon$}{
    $w_{t+1} \leftarrow 
        \Descend\big(w_t, \{ D_1\star\lambda_1 \dots D_K\star\lambda_K \} \big)$ \; 
    $c_k \leftarrow 
        \Cost\big(w_{t+1}, \{ D_k \}\big)~~\forall k$ \;  
    $\delta_k \leftarrow 
        \frac{1}{Z} \exp(\beta (c_k-r_k)) ~~ \forall k$ 
        \hfill\emph{--- with $Z$ such that $\sum_k\delta_k=1$} \;
    $\lambda_k \leftarrow \frac{1}{t+1}
        (t\lambda_k + \delta_k)  ~~ \forall k$ \;
    $t \leftarrow t+1$\;
}
\Return $w_t$
\caption{\label{alg:descentmain}
  A typical Lagrangian DRO algorithm. }
\end{algorithm}

\section{Calibration problems}
\label{sec:calibration}

As promised earlier, we now return to the statement of the simple DRO problem~\eqref{eqmain:dro} and discuss how it fails to properly account for many subtleties of the bias fighting problem.

One set of issues is described in the algorithmic bias literature~\citep{DBLP:journals/corr/LouizosSLWZ15,DBLP:journals/corr/BeutelCZC17, DBLP:journals/corr/abs-1801-07593, DBLP:conf/icml/HashimotoSNL18, DBLP:conf/aies/AminiSSBR19}. \emph{Representation disparity} refers to the pheonomenon of achieving a high overall accuracy but low minority accuracy. For instance, ubiquitous speech recognition systems, such as voice assistants, struggle with accents and dialects~\citep{DBLP:journals/taslp/BehravanHSKL16,DBLP:conf/icassp/YangARTRH18, DBLP:journals/speech/NajafianR20}. A minority user becomes discouraged by the poor performance of such system, which leads to \emph{disparity amplification} over time due to the increasing gap between quantity of data provided by active users (majority groups, favored by the system from the beginning) and groups that experienced poor performance due to the initial representation disparity. 
This shows that it is sometimes justified to augment the DRO statement with means to favor certain subpopulations in order to account for representation disparity or disparity amplification. A related development is the method proposed by Sagawa et al. \citep{sagawa-2020}: in order to account for the potentially small size of the training data for some distributions $P$, they augment the cost $C_P(w)$ with a penalty that decreases when the number of training examples for this distribution increases.

There might also be instances where it is justifiable to account for the difference in difficulty across distributions. For instance, it might be known that one of the training distribution represents examples collected with a deficient method, such as, bad cameras, bad conditions, etc.  Because the task is more difficult due to data limitations, the cost $C_P(w)$ for such a distribution will systematically be higher than for other distributions. The simple DRO formulation~\eqref{eqmain:dro} then amounts to optimizing only for this distribution. As a consequence, small gains for the deficient distribution will be obtained at the expense of a massive performance degradation for all other distributions, essentially making it as bad as the performance for the deficient distribution.  It might then be necessary to reduce the weight of this distribution in order to prevent it from dominating the DRO problem. In other words, the ideal solution to such a problem is to address the deficiencies of the data collection method.

Both issues (the need to account for pre-existing difficulties related to modelling certain subpopulations and the equalizing effect of DRO) can be addressed by augmenting the DRO problem with calibration constants that make such adjustments explicit. This leads to the calibrated DRO problem that we have already introduced when discussing Theorem~\ref{th:converse} (Calibrated DRO~\ref{eqmain:drocal}).
Specifying a set of calibration constants amounts to describing what we consider to be an \emph{acceptable} outcome for the original bias fighting problem. What is acceptable or not is obviously problem dependent and can be the object of difficult controversies. Making the calibration constants explicit separates the mathematical optimization statement from the difficult task of deciding what results are acceptable in a real-world problem.

A consequence of the mathematical theory is the practical duality between calibration constants $r_P$ and mixture coefficients $\lambda_{P_k}$. Theorem~\ref{th:one} says that a DRO local minimum for a particular choice of calibration constants is a stationary point of the expected loss for a particular mixture of the original distributions. Conversely, Theorem~\ref{th:converse} says that a local minimum of any particular expected risk mixture is also a DRO local minimum for a particular set of calibration constants. 

The calibration constants $r_P$ might in fact be a better way than mixture coefficients $\lambda_{P_k}$ to specify which performance discrepancies are considered acceptable across subpopulations because there are useful reference points for choosing them. The first reference point is to use equal calibration constants. DRO then optimizes the performance of the most adverse subpopulation at the cost of potentially degrading the performance of the system for the remaining subpopulations. See Figure~\ref{fig:ermdro} for an example. Another approach is to use the calibration constants $r_P^*$ that represent the best performance we can reach with our machine learning model on each distribution $P$ in isolation.
\[
    r_P^* = \min_w C_P(w)  
\]
Solving the DRO problem for these calibration constants amounts to constructing a single machine learning system that performs almost as well on each distribution $P$ as a dedicated machine learning system specifically trained for distribution $P$ alone. In the following section, we elaborate on this method of setting calibration constants.

Note that based on Theorems~\ref{th:one} and~\ref{th:converse}, regardless of the chosen calibration constants, no DRO solution can achieve a performance better than $r_P^*$ on any distribution $P$. If this were the case, it would mean that $r_P^*$ was not correctly estimated, and the new performance would become the corrected $r_P^*$.  This simple observation forms the basis for the practical recommendations discussed in the next section.

\section{A minimal set of practical recommendations}
\label{sec:recommendations}

In this section, we provide a minimal set of practical recommendations to machine learning engineers who face the difficult task of constructing and deploying bias-sensitive machine learning systems. We do not pretend that these recommendations are sufficient to address the bias problem, but merely represent intuitively sensible steps that are supported by our mathematical insights and should not be avoided. We summarize these recommendations in Inset~\ref{inset:recs}. 

We also motivate and elaborate on each step below.

\begin{inset}
    \centering
    \mdfdefinestyle{default}{}
    \begin{mdframed}[style=default, backgroundcolor=gray!15,
    frametitle={Fighting bias with DRO: practical recommendations}]
    \begin{enumerate}
    \item Identify subpopulations $P_k$ at risk in the available data.
    \item For \emph{each subpopulation, and in isolation}, determine the best performance $r_{P_k}^*$ that can be achieved with the machine learning model of choice.
    \item Decide whether the $r_{P_k}^*$ represent an acceptable set of performances. \emph{There is no point using DRO if this is not the case}. Instead, investigate why the model performs so poorly on the adverse distributions (insufficient data, inadequate model, etc.) until obtaining an acceptable set of $r_{P_k}^*$.
    \item Use DRO to construct a single machine learning system whose performance on each subpopulation is not much worse than $r_{P_k}^*$.
    This can be achieved by using the $r_{P_k}^*$ as calibration coefficients in a Lagrangian algorithm.
    \item Deploy the system on an experimental basis in order to collect more data. Sample the examples with the lowest accuracy in order to determine whether we missed a subpopulation at risk. If one is found, add the vulnerable subpopulation to the initial data and repeat all the steps.
    \end{enumerate}
    \medskip
    \end{mdframed} 
    \vspace{-1ex}
    \caption{Summary of practical recommendations.}
    \label{inset:recs}
    \vspace*{-2.4ex}
\end{inset}


The \emph{identification of the subpopulations} of concern frames the problem because it also defines the success criterion, that is, bias mitigation with respect to meaningful subpopulations. Key factors to consider are future users of the system, information on which groups have previously suffered from discrimination in similar scenarios, and the quantity and quality of the available data at the training time. In particular, we must at least have enough data to evaluate the subpopulation performances reliably. For instance, in a face recognition system, subpopulations might contain images of people representing distinct ethnicities~\citep{DBLP:journals/tifs/KlareBKBJ12}.  

Working \emph{on each subpopulation in isolation} attempts to determine the best achievable performance on each subpopulation if this subpopulation were the only target. Data available for minority subpopulations might be limited. In such case, data from remaining subpopulations can be used as a regularizer to improve performance on the subpopulation $P$ of interest. For instance, we can train on a mixture of data coming from both the subpopulation $P$ (with weight $1$) and the remaining subpopulations (with weights $\alpha_P$). We then treat $\alpha_P$ as a hyperparameter that we tune to achieve the best validation performance on data from the subpopulation $P$. Our estimate of $r^*_P$ is then the performance of the resulting system, either measured on the validation set, or on held out data if such data is available in sufficient quantity. This is why it is important to have sufficient data to reliably validate a model performance on each subpopulation. Techniques proposed to tackle noisy datasets and scenarios with limited labelled examples (active learning~\citep{ren2020survey}, transfer learning~\citep{DBLP:journals/tkde/PanY10,DBLP:journals/corr/abs-1808-01974}) can be used to increase the performance.

We can then judge whether the $r^*_P$ represent an \emph{acceptable set of performances} for a final system. As explained in Section~\ref{sec:calibration}, no DRO solution can perform better on a subpopulation $P$ than a model trained for this subpopulation $P$ only.
If the set of performances obtained in the previous steps is not acceptable, we must identify the \emph{root cause} of this problem. For instance, if poor performance stems from insufficient data quality for the subpopulation, this problem will persist at the step of finding a consistent system using DRO. We need to then focus on improving data quality for vulnerable subpopulations. We recommend investigating the root cause of insufficient performance for each of the vulnerable subpopulations in isolation.  

If minimum cost that can be achieved per each subpopulation is acceptable, we can then build a system that works consistently well across the subpopulations using DRO. In the simplest case, calibration coefficients $r_P$ per each subpopulation are going to be equal to the optimum expected risk for that subpopulation alone, $r_P = \min_{w} C_P(w)$. We can also adjust the calibration coefficients to prevent overfitting to individual subpopulations~\citep{sagawa-2020}. For $n$ examples in a certain subpopulation $P$, the expected risk $C_P(w)$ can be replaced by its empirical estimate $C_{P_n}(w)+$ augmented with a calibration constant that decreases when the number $n$ of training examples increases. Moreover, the model size often needs to be larger than the model size that achieves the best performance on each individual subpopulations. Intuitively, this is needed because handling all subpopulations at once might be more demanding than handling only one. In Section~\ref{sec:mixture} and Section A.4. in the Appendix, we also argue that overparametrization improves the issues associated with DRO local minima that are stationary points of an expected loss mixture but are not local minima of this mixture. As a result, overparametrization helps practical Lagragian DRO algorithms to find a good solution.

Finally, we must remain aware that the final system critically depends on the initial selection of the subpopulations of interest. Therefore, it remains essential to cautiously deploy such a system and to \emph{monitor} its performance during the ramp up. In particular, the worst performing cases should be examined for consistent patterns that might indicate that a vulnerable subpopulation was not considered in the problem specification. When this is the case, the correct solution is to include the initially omitted subpopulation and start again.


\section{DRO for adversarial examples}
\label{sec:adversarial}

The previous sections make several observations about the application of DRO to fight bias in machine learning systems.
In particular, we have argued that DRO is practically equivalent to training on a well chosen example distribution, and we have also shown that this well chosen example distribution is far from universal but depends on often implicit assumptions hidden in the DRO problem statement, such as calibration coefficients. 

These observations extend beyond the bias fighting scenario. For instance, DRO is often presented as a good way to construct systems that are robust to adversarial examples~\citep{szegedy-2014, madry2017towards}. This application of DRO can be formalized by considering a set $\Phi$ of all measurable functions $\varphi$ that map an example pattern $z$ to another pattern $\varphi(z)$
that is assumed to be \emph{visually indistinguishable} from $z$
according to a predefined criterion. For instance, it is common to consider the set of all transformations $\varphi$ such that $\|z-\varphi(z)\|_p\leq{\kappa}$, that is, transformations that can only modify an input pattern while remaining in a given $L_p$ ball.

Let $P_\varphi$ represent the distribution followed by $\varphi(z)$ when $z$ follows the distribution $P$. Robust solutions against the class of adversarial perturbation~$\Phi$ can be expressed as the DRO problem
\[
   \min_w \left\{ ~~~
    \max_{\varphi\in\Phi} C_{P_\varphi}(w) 
    ~~ = ~~
    \max_{P_\varphi\in\Q} C_{P_\varphi}(w) 
    ~~~ \right\} ~.
\]
The distribution family $\Q=\{P_\varphi:\varphi\in\Phi\}$ is typically much larger than the ones considered in the bias fighting scenario. Instead of representing a finite number of subpopulations, the family $\Q$ is usually infinite and  uncountable. Therefore one cannot reduce this DRO problem to a finitely constrained optimization problem and one cannot use a separation lemma because the family $\Q$ or its convex hull may not be topologically closed. 

Theorem~\ref{th:two} relies on an additional \emph{tightness} assumption\footnote{It is relatively easy to construct a sequence of finite mixtures whose gradient in $w^*$ tends to zero. The tightness assumption is useful to show that there exists a distribution whose gradient in $w^*$ is exactly zero.} to establish that a DRO local minimum $w^*$ is also a stationary point of the expected risk for an example distribution that belongs to the weak closure of the convex hull of $\Q$. The tightness assumption is trivially satisfied when the examples belong to a bounded domain (as is the case for images) or when they remain close to reference images drawn from a single distribution (as is the case for adversarial examples).

At first sight, this result seems to imply that there is a distribution of images on which the ordinary training procedure yields a solution robust to adversarial examples. Is it true that we would not have adversarial example issues if only we had the right examples to start with?  Or the right data augmentation scheme? 

More precisely, Theorem~\ref{th:two} states that a DRO local minimum is a stationary point of the expected risk for an example distribution that depends on all the details of the DRO problem and in particular on the definition of the set $\Phi$ of adversarial perturbation, or equivalently on which images are considered visually indistinguishable from a reference image. 
On one hand, we could use DRO with a class of adversarial perturbations $\Phi$ that are very conservatively below the threshold of visual distinguishability. For instance, the perturbation might be limited to changing pixel values by no more than a small threshold. Alas, the solution might be fooled by adversarial examples that do not satisfy this strict condition but nevertheless are still visually indistinguishable from the original pattern. On the other hand, we could use DRO with much broader class of perturbation, potentially including some that would affect a human observer. For instance, dithering patterns might occasionally introduce enough noise to be perceptually meaningful. Because such perturbations can dominate the DRO problem, it becomes necessary to introduce calibration constants in order to account for the variation in performance that can be justifiably expected with such perturbations.

Because DRO is fundamentally related to minimizing the expected cost for a well crafted example distribution, DRO does not really solve the original problem but displaces it into the specification of the class of adversarial perturbations and the selection of the associated cost calibration constants. However, the adversarial example scenario is substantially more challenging than the bias fighting scenario: because the number of potential perturbations is much larger than the number of potentially vulnerable subpopulations, we cannot work around the problem by first working on each of them in isolation as suggested in Section~\ref{sec:recommendations}. Using DRO for adversarial robustness without a reliable perceptual distance might be fundamentally flawed \citep{sharif-2018}.

\section{Related work}
Finding the appropriate choice for adversarial risk and making it match some notion of perceptual similarity is a topic of ongoing research. Early work on this topic considered $L_p$ norms as a similarity metric, for example $L_0$~\citep{DBLP:journals/corr/PapernotMJFCS15}, $L_2$~\citep{szegedy-2014}, or $L_\infty$~\citep{goodfellow2014explaining}. Sharif et al.~\citep{sharif-2018} show how these norms, as well as SSIM~\citep{journals/tip/WangBSS04}, fail to model perceptual similarity and still get fooled by simple adversarial examples. The view of adversarial machine learning through the lens of DRO is shared by Sinha et al.~\citep{DBLP:conf/iclr/SinhaND18}, who use Wasserstein distance as a measure of perceptual similarity and achieve important statistical guarantees regarding the computed solution, as well as excellent practical performance.

Rahimian and Mehrotra~\citep{rahimian2019distributionally} survey recent research in DRO, and in particular mention the various ways risk can be defined. To the best of our knowledge, the question of what the calibration coefficients should be has not been the topic of much investigation. Meinshausen et al.~\citep{meinshausen2015maximin} propose setting $r_P = \Var[Y_P]$ in order to maximize the minimum explained variance across distributions. Our suggestion is based on accounting for the acceptable performance (the best obtainable performance in isolation) of a particular subpopulation.

One thread in the debate on the source of bias was inspired by the outcomes of applying a photo upsampling algorithm~\citep{menon2020pulse} to images of non-white people. Examples of using DRO to approach similar problems include text autocomplete tasks~\citep{hashimoto2018fairness}, noisy minority subpopulations~\citep{DBLP:journals/corr/abs-2002-09343} and protection with respect to specific sensitive attributes~\citep{taskesen2020distributionally}, as well as lexical similarity and recidivism prediction~\citep{duchi2020distributionally}. The phenomeonon of neural networks exploiting 'shortcuts' in data~\citep{DBLP:journals/corr/abs-2004-07780} is a related line of work on robustness and fairness.

\section{Conclusion}
\label{sec:conclusion}
Whether fighting bias in machine learning systems is a data curation problem or an algorithmic problem has been the object of much discussion. Our theoretical results clarify the relation between a well known algorithmic approach, DRO, and the optimization of the expected cost on a well crafted data distribution. Contrary to the usual convex duality results, these results hold for nonconvex costs and for infinite families of distributions. These results also provide some support for the common practice of leveraging
this quasi-equivalence to design efficient DRO algorithms. But it also becomes clear that running such an imperfect DRO algorithms is equivalent to optimizing the expected risk for a well crafted distribution.

This analysis also makes clear that this well crafted distribution is not universal but depends on often implicit details of the DRO problem setup such as calibration constants. Alas, in the same way that there is no universally robust training set, there is no universal way to define calibration constants that ensure an acceptable set of results. However, an elementary argument shows that one cannot reach acceptable results with DRO unless one can reach acceptable results on each subpopulation in isolation. This forms the basis for a minimal set of practical recommendations. Finally we discuss how our insights ---this time with an infinite distribution family--- raise concerns about the commonly advocated use of DRO to tame adversarial examples without a reliable perceptual similarity criterion. 

Using DRO for fairness or adversarial robustness without a clear understanding of its algorithmic limitations can have a negative societal impact. Recommendations in Section~\ref{sec:recommendations} aim to prevent misuses of DRO, such as lowering performances on the remaining subpopulations to match the error on the most difficult distribution. However, as a consequence of Theorems~\ref{th:one}, \ref{th:two} and \ref{th:converse}, it is also necessary to address the underlying problems in the most challenging distribution. On on hand, failure to address the issues in the minority subpopulation leaves it susceptible to discrimination, both in the application at hand and in the future applications, where the unresolved issues might persist. On the other hand, reducing the performance of the majority populations can lead to an unacceptable average performance, and as a result, the system is not going to be used --- which might lead to a loss of interest in designing broadly accessible systems for this purpose (i.e., voice assistants robust to minority accents). We hope that our results and discussion will give more context to the debate on the sources of bias in machine learning, as well as help in bias mitigation in real-life scenarios.

\begin{ack}
We would like to thank Ferenc Huszár for valuable early feedback on this work.
%
\end{ack}

\bibliographystyle{plainnat}
\bibliography{neurips_2021}

\begin{thebibliography}{53}
\providecommand{\natexlab}[1]{#1}
\providecommand{\url}[1]{\texttt{#1}}
\expandafter\ifx\csname urlstyle\endcsname\relax
  \providecommand{\doi}[1]{doi: #1}\else
  \providecommand{\doi}{doi: \begingroup \urlstyle{rm}\Url}\fi

\bibitem[Amini et~al.(2019)Amini, Soleimany, Schwarting, Bhatia, and
  Rus]{DBLP:conf/aies/AminiSSBR19}
Alexander Amini, Ava~P. Soleimany, Wilko Schwarting, Sangeeta~N. Bhatia, and
  Daniela Rus.
\newblock {Uncovering and Mitigating Algorithmic Bias through Learned Latent
  Structure}.
\newblock In Vincent Conitzer, Gillian~K. Hadfield, and Shannon Vallor,
  editors, \emph{Proceedings of the 2019 {AAAI/ACM} Conference on AI, Ethics,
  and Society, {AIES} 2019, Honolulu, HI, USA, January 27-28, 2019}, pages
  289--295. {ACM}, 2019.
\newblock \doi{10.1145/3306618.3314243}.
\newblock URL \url{https://doi.org/10.1145/3306618.3314243}.

\bibitem[Angwin et~al.(2020)Angwin, Larson, Mattu, and Kirchner]{propublica}
Julia Angwin, Jeff Larson, Surya Mattu, and Lauren Kirchner.
\newblock {Machine bias: There’s software used across the country to predict
  future criminals. And it’s biased against blacks.}
\newblock \emph{ProPublica}, 2020.

\bibitem[Arjovsky et~al.(2019)Arjovsky, Bottou, Gulrajani, and
  Lopez{-}Paz]{arjovsky-2019}
Martin Arjovsky, L{\'{e}}on Bottou, Ishaan Gulrajani, and David Lopez{-}Paz.
\newblock {Invariant Risk Minimization}.
\newblock \emph{arXiv CoRR}, abs/1907.02893, 2019.
\newblock URL \url{http://arxiv.org/abs/1907.02893}.

\bibitem[Arrow et~al.(1958)Arrow, Hurwicz, and Uzawa]{uzawa-58}
K~Arrow, L.~Hurwicz, and H.~Uzawa.
\newblock \emph{{Studies in Nonlinear Programming}}.
\newblock Stanford Univ. Press, 1958.

\bibitem[Augustin et~al.(2020)Augustin, Meinke, and
  Hein]{augustin2020adversarial}
Maximilian Augustin, Alexander Meinke, and Matthias Hein.
\newblock {Adversarial robustness on in-and out-distribution improves
  explainability}.
\newblock In \emph{European Conference on Computer Vision}, pages 228--245.
  Springer, 2020.

\bibitem[Bagnell(2005)]{bagnell-2005}
J.~Andrew Bagnell.
\newblock {Robust Supervised Learning}.
\newblock In Manuela~M. Veloso and Subbarao Kambhampati, editors,
  \emph{Proceedings, The Twentieth National Conference on Artificial
  Intelligence and the Seventeenth Innovative Applications of Artificial
  Intelligence Conference, July 9-13, 2005, Pittsburgh, Pennsylvania, {USA}},
  pages 714--719. {AAAI} Press / The {MIT} Press, 2005.

\bibitem[Behravan et~al.(2016)Behravan, Hautam{\"{a}}ki, Siniscalchi, Kinnunen,
  and Lee]{DBLP:journals/taslp/BehravanHSKL16}
Hamid Behravan, Ville Hautam{\"{a}}ki, Sabato~Marco Siniscalchi, Tomi Kinnunen,
  and Chin{-}Hui Lee.
\newblock {i-Vector Modeling of Speech Attributes for Automatic Foreign Accent
  Recognition}.
\newblock \emph{{IEEE} {ACM} Trans. Audio Speech Lang. Process.}, 24\penalty0
  (1):\penalty0 29--41, 2016.
\newblock \doi{10.1109/TASLP.2015.2489558}.
\newblock URL \url{https://doi.org/10.1109/TASLP.2015.2489558}.

\bibitem[Ben{-}Tal et~al.(2009)Ben{-}Tal, Ghaoui, and
  Nemirovski]{DBLP:books/degruyter/Ben-TalGN09}
Aharon Ben{-}Tal, Laurent~El Ghaoui, and Arkadi Nemirovski.
\newblock \emph{{Robust Optimization}}, volume~28 of \emph{Princeton Series in
  Applied Mathematics}.
\newblock Princeton University Press, 2009.
\newblock ISBN 978-1-4008-3105-0.
\newblock \doi{10.1515/9781400831050}.
\newblock URL \url{https://doi.org/10.1515/9781400831050}.

\bibitem[Bertsekas(2009)]{bertsekas-2009}
D.P. Bertsekas.
\newblock \emph{Convex Optimization Theory}.
\newblock {Athena Scientific optimization and computation series}. Athena
  Scientific, 2009.

\bibitem[Beutel et~al.(2017)Beutel, Chen, Zhao, and
  Chi]{DBLP:journals/corr/BeutelCZC17}
Alex Beutel, Jilin Chen, Zhe Zhao, and Ed~H. Chi.
\newblock {Data Decisions and Theoretical Implications when Adversarially
  Learning Fair Representations}.
\newblock \emph{CoRR}, abs/1707.00075, 2017.
\newblock URL \url{http://arxiv.org/abs/1707.00075}.

\bibitem[Billingsley(1999)]{billingsley-1999}
Patrick Billingsley.
\newblock \emph{Convergence of probability measures}.
\newblock Wiley Series in Probability and Statistics: Probability and
  Statistics. John Wiley \& Sons Inc., New York, 1999.
\newblock ISBN 0-471-19745-9.

\bibitem[Blanchet et~al.(2019)Blanchet, Kang, and Murthy]{blanchet-2019}
Jose Blanchet, Yang Kang, and Karthyek Murthy.
\newblock {Robust {Wasserstein} profile inference and applications to machine
  learning}.
\newblock \emph{Journal of Applied Probability}, 56\penalty0 (3):\penalty0
  830--857, 2019.

\bibitem[{Bottou} et~al.(2018){Bottou}, {Curtis}, and
  {Nocedal}]{bottou-curtis-nocedal-2018}
L\'eon {Bottou}, Frank~E. {Curtis}, and Jorge {Nocedal}.
\newblock {Optimization Methods for Large-Scale Machine Learning}.
\newblock \emph{Siam Reviews}, 60\penalty0 (2):\penalty0 223--311, 2018.

\bibitem[Boyd and Vandenberghe(2014)]{boyd-vandenberghe-2004}
Stephen~P. Boyd and Lieven Vandenberghe.
\newblock \emph{{Convex Optimization}}.
\newblock Cambridge University Press, 2014.

\bibitem[Chouldechova(2017)]{DBLP:journals/bigdata/Chouldechova17}
Alexandra Chouldechova.
\newblock {Fair Prediction with Disparate Impact: {A} Study of Bias in
  Recidivism Prediction Instruments}.
\newblock \emph{Big Data}, 5\penalty0 (2):\penalty0 153--163, 2017.
\newblock \doi{10.1089/big.2016.0047}.
\newblock URL \url{https://doi.org/10.1089/big.2016.0047}.

\bibitem[Dastin(2018)]{amazon}
Jeffrey Dastin.
\newblock {Amazon scraps secret AI recruiting tool that showed bias against
  women}.
\newblock \emph{Reuters}, 2018.

\bibitem[Datta et~al.(2014)Datta, Tschantz, and
  Datta]{DBLP:journals/corr/DattaTD14}
Amit Datta, Michael~Carl Tschantz, and Anupam Datta.
\newblock {Automated Experiments on Ad Privacy Settings: {A} Tale of Opacity,
  Choice, and Discrimination}.
\newblock \emph{CoRR}, abs/1408.6491, 2014.
\newblock URL \url{http://arxiv.org/abs/1408.6491}.

\bibitem[Derman and Mannor(2020)]{derman2020distributional}
Esther Derman and Shie Mannor.
\newblock Distributional robustness and regularization in reinforcement
  learning.
\newblock \emph{arXiv preprint arXiv:2003.02894}, 2020.

\bibitem[Duchi et~al.(2020)Duchi, Hashimoto, and
  Namkoong]{duchi2020distributionally}
John Duchi, Tatsunori Hashimoto, and Hongseok Namkoong.
\newblock {Distributionally robust losses for latent covariate mixtures}.
\newblock \emph{arXiv preprint arXiv:2007.13982}, 2020.

\bibitem[Geirhos et~al.(2020)Geirhos, Jacobsen, Michaelis, Zemel, Brendel,
  Bethge, and Wichmann]{DBLP:journals/corr/abs-2004-07780}
Robert Geirhos, J{\"{o}}rn{-}Henrik Jacobsen, Claudio Michaelis, Richard~S.
  Zemel, Wieland Brendel, Matthias Bethge, and Felix~A. Wichmann.
\newblock {Shortcut Learning in Deep Neural Networks}.
\newblock \emph{CoRR}, abs/2004.07780, 2020.
\newblock URL \url{https://arxiv.org/abs/2004.07780}.

\bibitem[Goodfellow et~al.(2014{\natexlab{a}})Goodfellow, Shlens, and
  Szegedy]{goodfellow2014explaining}
Ian~J Goodfellow, Jonathon Shlens, and Christian Szegedy.
\newblock Explaining and harnessing adversarial examples.
\newblock \emph{arXiv preprint arXiv:1412.6572}, 2014{\natexlab{a}}.

\bibitem[Goodfellow et~al.(2014{\natexlab{b}})Goodfellow, Vinyals, and
  Saxe]{goodfellow2014qualitatively}
Ian~J Goodfellow, Oriol Vinyals, and Andrew~M Saxe.
\newblock {Qualitatively characterizing neural network optimization problems}.
\newblock \emph{arXiv preprint arXiv:1412.6544}, 2014{\natexlab{b}}.

\bibitem[Hashimoto et~al.(2018{\natexlab{a}})Hashimoto, Srivastava, Namkoong,
  and Liang]{hashimoto2018fairness}
Tatsunori Hashimoto, Megha Srivastava, Hongseok Namkoong, and Percy Liang.
\newblock {Fairness without demographics in repeated loss minimization}.
\newblock In \emph{International Conference on Machine Learning}, pages
  1929--1938. PMLR, 2018{\natexlab{a}}.

\bibitem[Hashimoto et~al.(2018{\natexlab{b}})Hashimoto, Srivastava, Namkoong,
  and Liang]{DBLP:conf/icml/HashimotoSNL18}
Tatsunori~B. Hashimoto, Megha Srivastava, Hongseok Namkoong, and Percy Liang.
\newblock {Fairness Without Demographics in Repeated Loss Minimization}.
\newblock In Jennifer~G. Dy and Andreas Krause, editors, \emph{Proceedings of
  the 35th International Conference on Machine Learning, {ICML} 2018,
  Stockholmsm{\"{a}}ssan, Stockholm, Sweden, July 10-15, 2018}, volume~80 of
  \emph{Proceedings of Machine Learning Research}, pages 1934--1943. {PMLR},
  2018{\natexlab{b}}.

\bibitem[Hu et~al.(2018)Hu, Niu, Sato, and Sugiyama]{hu-2018}
Weihua Hu, Gang Niu, Issei Sato, and Masashi Sugiyama.
\newblock {Does Distributionally Robust Supervised Learning Give Robust
  Classifiers?}
\newblock In Jennifer Dy and Andreas Krause, editors, \emph{Proceedings of the
  35th International Conference on Machine Learning}, volume~80 of
  \emph{Proceedings of Machine Learning Research}, pages 2029--2037, 2018.

\bibitem[Klare et~al.(2012)Klare, Burge, Klontz, Bruegge, and
  Jain]{DBLP:journals/tifs/KlareBKBJ12}
Brendan Klare, Mark~James Burge, Joshua~C. Klontz, Richard W.~Vorder Bruegge,
  and Anil~K. Jain.
\newblock {Face Recognition Performance: Role of Demographic Information}.
\newblock \emph{{IEEE} Trans. Inf. Forensics Secur.}, 7\penalty0 (6):\penalty0
  1789--1801, 2012.
\newblock \doi{10.1109/TIFS.2012.2214212}.
\newblock URL \url{https://doi.org/10.1109/TIFS.2012.2214212}.

\bibitem[Kuhn et~al.(2019)Kuhn, Esfahani, Nguyen, and
  Shafieezadeh-Abadeh]{kuhn2019wasserstein}
Daniel Kuhn, Peyman~Mohajerin Esfahani, Viet~Anh Nguyen, and Soroosh
  Shafieezadeh-Abadeh.
\newblock {Wasserstein distributionally robust optimization: Theory and
  applications in machine learning}.
\newblock In \emph{Operations Research \& Management Science in the Age of
  Analytics}, pages 130--166. INFORMS, 2019.

\bibitem[Louizos et~al.(2016)Louizos, Swersky, Li, Welling, and
  Zemel]{DBLP:journals/corr/LouizosSLWZ15}
Christos Louizos, Kevin Swersky, Yujia Li, Max Welling, and Richard~S. Zemel.
\newblock {The Variational Fair Autoencoder}.
\newblock In Yoshua Bengio and Yann LeCun, editors, \emph{4th International
  Conference on Learning Representations, {ICLR} 2016, San Juan, Puerto Rico,
  May 2-4, 2016, Conference Track Proceedings}, 2016.
\newblock URL \url{http://arxiv.org/abs/1511.00830}.

\bibitem[Madry et~al.(2017)Madry, Makelov, Schmidt, Tsipras, and
  Vladu]{madry2017towards}
Aleksander Madry, Aleksandar Makelov, Ludwig Schmidt, Dimitris Tsipras, and
  Adrian Vladu.
\newblock Towards deep learning models resistant to adversarial attacks.
\newblock \emph{arXiv preprint arXiv:1706.06083}, 2017.

\bibitem[Meinshausen and B\"{u}hlmann(2015)]{meinshausen-buhlman-2015}
Nicolai Meinshausen and Peter B\"{u}hlmann.
\newblock {Maximin effects in inhomogeneous large-scale data}.
\newblock \emph{The Annals of Statistics}, 43\penalty0 (4):\penalty0
  1801--1830, 2015.

\bibitem[Meinshausen et~al.(2015)Meinshausen, B{\"u}hlmann,
  et~al.]{meinshausen2015maximin}
Nicolai Meinshausen, Peter B{\"u}hlmann, et~al.
\newblock {Maximin effects in inhomogeneous large-scale data}.
\newblock \emph{The Annals of Statistics}, 43\penalty0 (4):\penalty0
  1801--1830, 2015.

\bibitem[Menon et~al.(2020)Menon, Damian, Hu, Ravi, and Rudin]{menon2020pulse}
Sachit Menon, Alexandru Damian, Shijia Hu, Nikhil Ravi, and Cynthia Rudin.
\newblock {PULSE: Self-supervised photo upsampling via latent space exploration
  of generative models}.
\newblock In \emph{Proceedings of the IEEE/CVF Conference on Computer Vision
  and Pattern Recognition}, pages 2437--2445, 2020.

\bibitem[Metz and Satariano(2020)]{nyt}
Cade Metz and Adam Satariano.
\newblock {An Algorithm That Grants Freedom, or Takes It Away}.
\newblock \emph{The New York Times}, 2020.
\newblock ISSN 0362-4331.

\bibitem[Najafian and Russell(2020)]{DBLP:journals/speech/NajafianR20}
Maryam Najafian and Martin~J. Russell.
\newblock {Automatic accent identification as an analytical tool for accent
  robust automatic speech recognition}.
\newblock \emph{Speech Commun.}, 122:\penalty0 44--55, 2020.
\newblock \doi{10.1016/j.specom.2020.05.003}.
\newblock URL \url{https://doi.org/10.1016/j.specom.2020.05.003}.

\bibitem[Namkoong and Duchi(2016)]{namkoong-duchy-2016}
Hongseok Namkoong and John~C. Duchi.
\newblock {Stochastic Gradient Methods for Distributionally Robust Optimization
  with f-divergences}.
\newblock In Daniel~D. Lee, Masashi Sugiyama, Ulrike von Luxburg, Isabelle
  Guyon, and Roman Garnett, editors, \emph{Advances in Neural Information
  Processing Systems 29: Annual Conference on Neural Information Processing
  Systems 2016, December 5-10, 2016, Barcelona, Spain}, pages 2208--2216, 2016.

\bibitem[Pan and Yang(2010)]{DBLP:journals/tkde/PanY10}
Sinno~Jialin Pan and Qiang Yang.
\newblock {A Survey on Transfer Learning}.
\newblock \emph{{IEEE} Trans. Knowl. Data Eng.}, 22\penalty0 (10):\penalty0
  1345--1359, 2010.
\newblock \doi{10.1109/TKDE.2009.191}.
\newblock URL \url{https://doi.org/10.1109/TKDE.2009.191}.

\bibitem[Papernot et~al.(2015)Papernot, McDaniel, Jha, Fredrikson, Celik, and
  Swami]{DBLP:journals/corr/PapernotMJFCS15}
Nicolas Papernot, Patrick~D. McDaniel, Somesh Jha, Matt Fredrikson, Z.~Berkay
  Celik, and Ananthram Swami.
\newblock {The Limitations of Deep Learning in Adversarial Settings}.
\newblock \emph{CoRR}, abs/1511.07528, 2015.
\newblock URL \url{http://arxiv.org/abs/1511.07528}.

\bibitem[Qian et~al.(2020)Qian, Alaa, and van~der Schaar]{NEURIPS2020_79a3308b}
Zhaozhi Qian, Ahmed~M. Alaa, and Mihaela van~der Schaar.
\newblock {When and How to Lift the Lockdown? Global COVID-19 Scenario Analysis
  and Policy Assessment using Compartmental Gaussian Processes}.
\newblock In H.~Larochelle, M.~Ranzato, R.~Hadsell, M.~F. Balcan, and H.~Lin,
  editors, \emph{Advances in Neural Information Processing Systems}, volume~33,
  pages 10729--10740. Curran Associates, Inc., 2020.
\newblock URL
  \url{https://proceedings.neurips.cc/paper/2020/file/79a3308b13cd31f096d8a4a34f96b66b-Paper.pdf}.

\bibitem[Rahimian and Mehrotra(2019)]{rahimian2019distributionally}
Hamed Rahimian and Sanjay Mehrotra.
\newblock {Distributionally robust optimization: A review}.
\newblock \emph{arXiv preprint arXiv:1908.05659}, 2019.

\bibitem[Rahmattalabi et~al.(2019)Rahmattalabi, Vayanos, Fulginiti, Rice,
  Wilder, Yadav, and Tambe]{DBLP:conf/nips/RahmattalabiVFR19}
Aida Rahmattalabi, Phebe Vayanos, Anthony Fulginiti, Eric Rice, Bryan Wilder,
  Amulya Yadav, and Milind Tambe.
\newblock {Exploring Algorithmic Fairness in Robust Graph Covering Problems}.
\newblock In Hanna~M. Wallach, Hugo Larochelle, Alina Beygelzimer, Florence
  d'Alch{\'{e}}{-}Buc, Emily~B. Fox, and Roman Garnett, editors, \emph{Advances
  in Neural Information Processing Systems 32: Annual Conference on Neural
  Information Processing Systems 2019, NeurIPS 2019, December 8-14, 2019,
  Vancouver, BC, Canada}, pages 15750--15761, 2019.
\newblock URL
  \url{https://proceedings.neurips.cc/paper/2019/hash/1d7c2aae840867027b7edd17b6aaa0e9-Abstract.html}.

\bibitem[Ren et~al.(2020)Ren, Xiao, Chang, Huang, Li, Chen, and
  Wang]{ren2020survey}
Pengzhen Ren, Yun Xiao, Xiaojun Chang, Po-Yao Huang, Zhihui Li, Xiaojiang Chen,
  and Xin Wang.
\newblock {A Survey of Deep Active Learning}, 2020.

\bibitem[Sagawa et~al.(2020)Sagawa, Koh, Hashimoto, and Liang]{sagawa-2020}
Shiori Sagawa, Pang~Wei Koh, Tatsunori~B. Hashimoto, and Percy Liang.
\newblock {Distributionally Robust Neural Networks}.
\newblock In \emph{International Conference on Learning Representations}, 2020.
\newblock URL \url{https://openreview.net/forum?id=ryxGuJrFvS}.

\bibitem[Sagun et~al.(2018)Sagun, Evci, G\"{u}ney, Dauphin, and
  Bottou]{sagun-2018}
Levent Sagun, Utku Evci, Veli~U\u{g}ur G\"{u}ney, Yann Dauphin, and L\'{e}on
  Bottou.
\newblock {Empirical Analysis of the {Hessian} of Over-Parametrized Neural
  Networks}.
\newblock In \emph{Sixth International Conference on Learning Representations
  (ICLR), Workshop paper}, 2018.
\newblock URL \url{http://leon.bottou.org/papers/sagun-2018}.

\bibitem[Sharif et~al.(2018)Sharif, Bauer, and Reiter]{sharif-2018}
Mahmood Sharif, Lujo Bauer, and Michael~K. Reiter.
\newblock {On the Suitability of Lp-Norms for Creating and Preventing
  Adversarial Examples}.
\newblock In \emph{Proceedings of the IEEE Conference on Computer Vision and
  Pattern Recognition (CVPR) Workshops}, June 2018.

\bibitem[Sinha et~al.(2018)Sinha, Namkoong, and
  Duchi]{DBLP:conf/iclr/SinhaND18}
Aman Sinha, Hongseok Namkoong, and John~C. Duchi.
\newblock {Certifying Some Distributional Robustness with Principled
  Adversarial Training}.
\newblock In \emph{6th International Conference on Learning Representations,
  {ICLR} 2018, Vancouver, BC, Canada, April 30 - May 3, 2018, Conference Track
  Proceedings}, 2018.
\newblock URL \url{https://openreview.net/forum?id=Hk6kPgZA-}.

\bibitem[Staib and Jegelka(2019)]{staib-jegelka-2019}
Matthew Staib and Stefanie Jegelka.
\newblock {Distributionally Robust Optimization and Generalization in Kernel
  Methods}.
\newblock In Hanna~M. Wallach, Hugo Larochelle, Alina Beygelzimer, Florence
  d'Alch{\'{e}}{-}Buc, Emily~B. Fox, and Roman Garnett, editors, \emph{Advances
  in Neural Information Processing Systems 32: Annual Conference on Neural
  Information Processing Systems 2019, NeurIPS 2019, December 8-14, 2019,
  Vancouver, BC, Canada}, pages 9131--9141, 2019.

\bibitem[Szegedy et~al.(2014)Szegedy, Zaremba, Sutskever, Bruna, Erhan,
  Goodfellow, and Fergus]{szegedy-2014}
Christian Szegedy, Wojciech Zaremba, Ilya Sutskever, Joan Bruna, Dumitru Erhan,
  Ian Goodfellow, and Rob Fergus.
\newblock Intriguing properties of neural networks.
\newblock In \emph{International Conference on Learning Representations}, 2014.
\newblock URL \url{https://openreview.net/forum?id=kklr_MTHMRQjG}.

\bibitem[Tan et~al.(2018)Tan, Sun, Kong, Zhang, Yang, and
  Liu]{DBLP:journals/corr/abs-1808-01974}
Chuanqi Tan, Fuchun Sun, Tao Kong, Wenchang Zhang, Chao Yang, and Chunfang Liu.
\newblock {A Survey on Deep Transfer Learning}.
\newblock \emph{CoRR}, abs/1808.01974, 2018.
\newblock URL \url{http://arxiv.org/abs/1808.01974}.

\bibitem[Taskesen et~al.(2020)Taskesen, Nguyen, Kuhn, and
  Blanchet]{taskesen2020distributionally}
Bahar Taskesen, Viet~Anh Nguyen, Daniel Kuhn, and Jose Blanchet.
\newblock {A distributionally robust approach to fair classification}.
\newblock \emph{arXiv preprint arXiv:2007.09530}, 2020.

\bibitem[Wang et~al.(2020)Wang, Guo, Narasimhan, Cotter, Gupta, and
  Jordan]{DBLP:journals/corr/abs-2002-09343}
Serena Wang, Wenshuo Guo, Harikrishna Narasimhan, Andrew Cotter, Maya~R. Gupta,
  and Michael~I. Jordan.
\newblock {Robust Optimization for Fairness with Noisy Protected Groups}.
\newblock \emph{CoRR}, abs/2002.09343, 2020.
\newblock URL \url{https://arxiv.org/abs/2002.09343}.

\bibitem[Wang et~al.(2004)Wang, Bovik, Sheikh, and
  Simoncelli]{journals/tip/WangBSS04}
Zhou Wang, Alan~C. Bovik, Hamid~R. Sheikh, and Eero~P. Simoncelli.
\newblock Image quality assessment: from error visibility to structural
  similarity.
\newblock \emph{IEEE Transactions on Image Processing}, 13\penalty0
  (4):\penalty0 600--612, 2004.
\newblock URL
  \url{http://dblp.uni-trier.de/db/journals/tip/tip13.html#WangBSS04}.

\bibitem[Yang et~al.(2018)Yang, Audhkhasi, Rosenberg, Thomas, Ramabhadran, and
  Hasegawa{-}Johnson]{DBLP:conf/icassp/YangARTRH18}
Xuesong Yang, Kartik Audhkhasi, Andrew Rosenberg, Samuel Thomas, Bhuvana
  Ramabhadran, and Mark Hasegawa{-}Johnson.
\newblock {Joint Modeling of Accents and Acoustics for Multi-Accent Speech
  Recognition}.
\newblock In \emph{2018 {IEEE} International Conference on Acoustics, Speech
  and Signal Processing, {ICASSP} 2018, Calgary, AB, Canada, April 15-20,
  2018}, pages 5989--5993. {IEEE}, 2018.
\newblock \doi{10.1109/ICASSP.2018.8462557}.
\newblock URL \url{https://doi.org/10.1109/ICASSP.2018.8462557}.

\bibitem[Zhang et~al.(2018)Zhang, Lemoine, and
  Mitchell]{DBLP:journals/corr/abs-1801-07593}
Brian~Hu Zhang, Blake Lemoine, and Margaret Mitchell.
\newblock {Mitigating Unwanted Biases with Adversarial Learning}.
\newblock \emph{CoRR}, abs/1801.07593, 2018.
\newblock URL \url{http://arxiv.org/abs/1801.07593}.

\end{thebibliography}

\newpage
\appendix

\section{Theoretical Appendix}
\setcounter{theorem}{0}
\setcounter{definition}{0}
\setcounter{lemma}{0}

\paragraph{Notation}
Let $\ell(z,w)$ be the loss of a machine learning model where $w\in\R^d$ represent the parameters of the model and $z\in\R^n$ belongs to the space of examples. For instance, in the case of least square regression, the examples $z$ are pairs $(x,y)$ and the loss is $\ell(z,w)=\|y-f_w(x)\|^2$.

\paragraph{Distribution Robust Optimization (DRO)}
Instead of assuming the existence of a probability distribution $P(z)$ over the examples $z$ and formulating an Expected Risk Minimization (ERM) problem: 
\begin{equation}
    \label{eq:risk}
    \min_w ~~ \left\{ ~ C_P(w) = E_{z\sim{P}} [\ell(z,w)] ~ \right\},
\end{equation}
the Distribution Robust Optimization (DRO) problem considers a
family $\Q$ of distributions and seeks to minimize
\begin{equation}
\label{eq:dro}
    \min_w ~~ \left\{\quad C_\Q(w) 
      ~\stackrel{\Delta}{=}~ \max_{P\in\Q} C_P(w)
      ~=~ \max_{P\in\Q} E_{z\sim{P}} [\ell(z,w)] 
      \quad\right\}.
\end{equation}
Many authors define $\Q$ with the purpose of constructing a learning algorithm with additional robustness properties. For instance, $\Q$ may be the set of all distributions located within a certain distance of the training distribution \citep{bagnell-2005,namkoong-duchy-2016,blanchet-2019,staib-jegelka-2019}. Different ways to measure this distance lead to different and sometimes surprising solutions \citep[\eg.,][]{hu-2018}.  Interesting theoretical possibilities appear when $\Q$ also contains the discrete distributions that represent finite training sets.  Besides these theoretically justified choices of $\Q$, many practical concerns can be viewed through the prism of DRO on ad-hoc families $\Q$ of distributions. 

\begin{example}[Fighting bias]
\label{ex:bias}
Let the example distributions $P_1$ to $P_K$ represent identified subpopulations for which we want to ensure consistent performance. It is appealing to formulate this problem as a DRO problem with \mbox{$\Q=\{P_1\dots P_K\}$}.
However, as discussed in the main text, it is important
to realize that much of the original problem is hiding
behind the choice of the calibration constants.
\end{example}

\begin{example}[Fighting adversarial attacks]
\label{ex:adversarial}
\citet{szegedy-2014} have shown that one can almost arbitrarily change the output of a deep learning vision system by modifying the patterns in nearly invisible ways. Let $\Phi$ be the set of all measurable functions $\varphi$ that map an example pattern $z$ to another pattern $\varphi(z)$ that is assumed visually indistinguishable from $z$ according to a certain psycho-visual criterion. Let $P_\varphi$ represent the distribution followed by $\varphi(z)$ when $z$ follows the distribution $P$. Robust solutions against the class of adversarial perturbation~$\Phi$ can be found with DRO with the distribution family \mbox{$\Q=\{P_\varphi:\varphi\in\Phi\}$}.
\end{example}

\paragraph{Calibrated costs}
The simple DRO formulation makes sense when we know that all distributions define problems of comparable difficulty. It is however easy to imagine that a particular distribution emphasises harder examples.
We can introduce calibration terms $r_P$ in the DRO formulation to prevent any single distribution $P$ to dominate the maximum.
\begin{equation}
    \label{eq:drocal}
        \min_w ~~ \left\{\quad C_{\Q,r}(w) 
      ~\stackrel{\Delta}{=}~ \max_{P\in\Q} 
          \{ C_P(w) - r_P \}
        \quad\right\}.
\end{equation}
Correctly setting the calibration terms is both difficult and application specific. A simple but costly approach consists in letting $r_P$ be equal to the optimum cost for that distribution alone, $r_P = \min_{w} C_P(w)$.
Calibrated DRO \eqref{eq:drocal} then controls the loss of performance incurred by seeking a solution that works for all distributions as opposed to solutions that are specific to each distribution. Another approach \citep{meinshausen-buhlman-2015} relies instead on the variance of the predicted quantity.

Calibration terms can also be used to counter the effect of finite training data. For instance, when we only have $n$ examples for a certain distribution $P\in\Q$, the expected risk $C_P(w)$ can be replaced by its empirical estimate $C_{P_n}(w)+$ augmented with a calibration constant that decreases when the number $n$ of training examples increases~\citep{sagawa-2020}.

\subsection{A local minimum of a DRO problem is 
            a stationary point of an expected loss mixture}

\subsubsection{Finite case}
We first address the case where $\Q$ is a
finite set of distributions $P_1\dots{P_K}$.
The following result simplifies
Proposition~2 of \citet{arjovsky-2019}
by eliminating the KKT constraint qualification 
requirement. In the rest of this document, we always assume 
that the mixture coefficients $\lambda_k$ are nonnegative and sum to one.

\begin{xreftheorem}{th:one}
Let $\Q=\{P_1,\dots,P_K\}$ be a finite set of probability distributions on~$\R^n$ and let $w^*$ be a local minimum of the DRO problem~\eqref{eq:dro} or the calibrated DRO problem \eqref{eq:drocal}. Let the costs $C_P(w)=\E_{z\sim P}[\ell(z,w)]$ be differentiable in $w$ for all $P\in\Q$. Then there exists a mixture distribution $\Pmix=\sum_{k}\lambda_k P_k$ such that $\nabla{C}_{\Pmix}(w^*)=0$.
\end{xreftheorem}

\noindent
The proof relies on a simple hyperplane separation lemma closely related to Farkas' lemma~\citep[Sec.2.5 and Ex.2.20]{boyd-vandenberghe-2004}.

\begin{lemma}
\label{th:farkas}
A nonempty closed convex subset $A$ of $\R^n$ either contains the origin or is strictly separated from the origin by a certain hyperplane, that is, there exists a vector $u\in\R^n$ and a scalar $c>0$ such that, for all $x\in{A}$, $\dotp{u}{x}\geq{c}$.
\end{lemma}

\begin{proof}
Assume $0\notin{A}$. Let $u\in{A}$ be the projection of the origin onto the closed convex set $A$. For all $x\in{A}$ and all $t\in[0,1]$, the point $u+t(x-u)$ also belongs to the convex set $A$. Since $u$ is the point of $A$ closest to the origin,
\[
   \forall t\in[0,1] ~~
   r(t)=\|u+t(x-u)\|^2 
    = \|u\|^2 +2t\dotp{u}{x-u}+t^2\|x-u\|^2
    \geq \|u|^2.
\]
Therefore $r'(0)=2\dotp{u}{x-u}\geq0$, that is, $\dotp{u}{x}\geq\dotp{u}{u}>0$.
\end{proof}

\begin{proof}[Proof of Theorem~\ref{th:one}]
Let $A\subset\R^n$ be the convex hull of the  $g_k{=}\nabla{C}_{P_k}(w^*)$ for $k=1\dots{K}$. $A$ is closed and convex. If $A$ does not contain the origin, according to the lemma, there exist $u$ and $c$ such that $\forall~x\in{A}$, $\dotp{u}{x}\geq{c}>0$. Therefore, for all $t>0$, moving from $w^*$ to $w^*-tu$ reduces all costs $C_{P_k}$ by at least $tc+o(t)$. As a consequence, $\max_k C_{P_k}$ is also reduced by at least $tc+o(t)$, contradicting the assumption that $w^*$ is a local minimum. 
Hence $A$ contains the origin, which means that there are positive mixture coefficients $\lambda_k$ summing to one such that $\sum_k\lambda_k\nabla{C}_{P_k}(w)= \nabla_w C_\Pmix(w) = 0.$
\end{proof}

All local and global solutions of the DRO problem~\eqref{eq:dro} or~\eqref{eq:drocal} are therefore stationary points of the expected risk~\eqref{eq:risk} associated with a mixture of the distributions of $\Q$. 
The exact mixture coefficients depend on the loss functions,
the distributions included in $\Q$ and, in the case of the calibrated version of DRO, on the calibration constants $r_P$. 

This result raises several important questions. Is this result valid when $\Q$ is not finite? Are these stationary points always local minima? Is the converse true? What is the relation between the mixture coefficients $\lambda_k$ and the calibration constants $r_P$? How far can such results go
without assuming convex losses? These questions will be
addressed in the rest of this document.

\subsubsection{Infinite case}

The infinite case differs because the convex hull of an infinite set of vectors is not necessarily closed, even when the original set is closed. Therefore cannot directly apply the lemma to the convex hull $A$ of the gradients $g_P=\nabla{C}_P(w^*)$ for all $P\in\Q$. Appling instead the lemma to the closure $\bar{A}$ of $A$ yields a substantially weaker result: if $w^*$ is a local DRO minimum, then for each $\varepsilon>0$, there is a mixture ${\Pmixeps}$ of distributions from~$\Q$ such that
\mbox{$\|\nabla{C}_{\smash{\Pmixeps}}\|\le\varepsilon$}. 

Note there is no guarantee that ${\Pmixeps}$ converges to an actual distribution when $\varepsilon$ converges to zero.\footnote{Suppose for instance that $\Q$ contains all Gaussians with unit variance with arbitrary means in $\mathbb{R}$. For any $t>0$, let ${P_{\mathrm{mix}}^{(1/t)}}$ be the equal mixture of $t^2$ equally spaced Gaussians in interval $[-t,+t]$. Neither this sequence not any of its subsequences converge to a distribution because there is no such thing as a uniform distribution on all of $\mathbb{R}$.} Therefore this weaker result does not help relating the solution of a DRO problem with the solutions of an ERM problem for a suitable training distribution. However such a stronger result
can be obtained at the price of a \emph{tightness} assumption~\citep{billingsley-1999}.
\begin{definition}
\label{th:tight}
A family of distributions $\Q$ on a Polish space\footnote{For our purposes, it is sufficient to know that $\R^n$ is a Polish space!} $\Omega$ is \emph{tight} when, for any $\epsilon>0$, there is a compact subset $K\subset\Omega$ such that $\forall P\in\Q$, $P(K)\geq1-\epsilon$.
\end{definition}
\noindent
Tightness is therefore obvious when all the examples belong to a bounded domain. Even when this is not the case, it is known that any finite set of probability distributions on a Polish space is~tight~\citep{billingsley-1999}. This often provides the means to prove the tightness of an infinite family $\Q$ of distributions that are "close" enough to a single distribution such as the training data distribution. For instance, in the case of adversarial examples (Example~\ref{ex:adversarial}), tightness is doubly obvious, first because all images belong to a bounded domain, second because the visual similarity criterion ensures that the distance between $z$ and $\varphi(z)$ is bounded.
 
\begin{xreftheorem}{th:two}
Let $\Q$ be a tight family of probability distributions on $\R^n$. Let $w^*$ be a local minimum of problem~\eqref{eq:drocal}. Let $\Qmix$ be the weak convergence closure of the convex hull of $\Q$. Let there be a bounded continuous function $h(z,w)$ defined on a neighborhood $\mathcal{V}$ of $w^*$ such that $\nabla{C}_P(w)=\E_{z\sim{P}}[h(z,w)]$ for all $P\in\Qmix$ and such that $\|h(z,w)-h(z,w')\|\leq{M}\|w-w'\|$ for almost all $z\in\R^n$. Then $\Qmix$ contains a distribution $\Pmix$ such that $\nabla_w C_\Pmix(w^*)=0$.
\end{xreftheorem}

Following~\cite{bottou-curtis-nocedal-2018}, the theorem does not require the loss $\ell(z,w)$ to be differentiable everywhere as long as the purported derivative $h(z,w)$ has the correct expectation. For our purposes, it must also be bounded and continuous on $\mathcal{V}$ and satisfy a Lipschitz continuity requirement. 

\begin{proof}
Let $\bar{A}$ be the closure of the convex hull of the
$g_P=\nabla C_P(w^*)$ for all $P\in\Q$. According to Lemma~\ref{th:farkas}, if $\bar{A}$ does not contain the origin, then there are $u$ and $c>0$ such that $\forall x\in{A}, \dotp{u}{x}>c$. In particular, for all $P\in\Q$, we have
$\dotp{u}{\nabla{C_P}(w^*)}>c>0$. Thanks to the Lipschitz continuity of $h(z,w)$, we have $C_P(w^*-tu)<C_P(w^*)-tc+Mt^2$ for all $P\in\Q$. Therefore for any $0<t<c/2M$ and any $P\in\Q$,
we have $C_P(w^*-tu)<C_P(w^*)-tc/2$ contradicting the assumption that $w^*$ is a local DRO miminum. Therefore $\bar{A}$ contains the origin. This means that for any $t>0$, there exists a mixture $\Pmixoot$ of distributions from $\Q$ such that $\|\nabla{C_{\smash{\Pmixeps}}}(w^*)\|<1/t$.
Note that if $\Q$ is tight, the convex hull of $\Q$ is also tight. Therefore the sequence $\Pmixoot$ is also tight,
and, by Prokhorov's theorem, contains a weakly convergent subsequence whose limit $\Pmix$ belongs to the closure $\Qmix$ of
the convex hull of $\Q$. Because $h(z,w^*)$ is continuous and bounded, the map $P\mapsto\nabla{C_P}(w^*)$ is continuous for the weak topology. Therefore $\nabla{C_\Pmix})(w^*)=0$.
\end{proof}

\subsection{A local minimum of an expected loss mixture is 
a local minimum of a calibrated DRO problem}

The following elementary result states that if $w^*$ is a local minimum of an expected cost mixture $C_\Pmix$, then it also is a local minimum of the calibrated DRO problem~\eqref{eq:drocal} with calibration constants $r_P$ equal to the costs $C_P(w^*)$.

\begin{theorem}[Converse]
\label{th:converseappendix}
Let $\Pmix=\sum_k\lambda_kP_k$ be an
arbitrary mixture of distributions $P_k\in\Q$. If $w^*$ is a local minimum of $C_\Pmix$, then $w^*$ is a local minimum of the calibrated DRO problem~\eqref{eq:drocal} with calibration coefficients $r_P=C_P(w^*)$.
\end{theorem}

\begin{proof}
By contradiction, assume that $w^*$ is not a local minimum of~\eqref{eq:drocal}, that is, for all $\epsilon>0$ there exists $u$ such that $\|u\|<\epsilon$ and
$
  \max_{P\in\Q}\left\{ C_{P}(w^*+u)-r_{P} \right\}
  < \max_{P\in\Q}\left\{ C_{P}(w^*)-r_{P} \right\}
$.
Recalling our choice of $r_P$ yields $\max_{P\in\Q}\left\{ C_{P}(w^*+u)-C_{P}(w^*) \right\} < 0$. Since $C_{P}(w*+u) < C_{P}(w^*)$ for all $P\in\Q$, $C_\Pmix(w^*+u) <C_\Pmix(w^*)$, and~$w^*$ cannot be a local minimum of~$C_\Pmix$.
\end{proof}

\subsection{With and without convexity assumptions}

Note that there is a discrepancy between the statements of Theorem~\ref{th:converse} and Theorems~\ref{th:one}-\ref{th:two}. The former requires a local minimum of the expected loss mixture, whereas the latter only provides a stationary point. This distinction becomes moot if we assume that the loss functions $\ell(z,w)$ are convex in $w$. When this is the case, all stationary points are not only local minima, but global minima as well.\footnote{Convexity also provides easy means to weaken the differentiability assumption because of the existence of subgradients. One could similarly weaken the differentiability assumptions of Theorems~\ref{th:one}-\ref{th:two} by assuming instead the existence of local sub-- and super--gradients.}
Theorems~\ref{th:one} and theorem~\ref{th:converse} then 
provide an exact equivalence between finding a minimum of
the calibrated DRO problem~\eqref{eq:drocal} and finding
a minimum of an expected loss mixture.

However there is little point in providing theorems for the convex case because, at least in the finite case, it is well covered by theory of convex duality~\citep{bertsekas-2009} applied to a simple restatement of convex DRO as a convex optimization problem with an additional slack variable $L$
\[
    \min_{w,L}  L  \quad \text{s.t.} \quad
        \forall P\in\Q~~ C_P(w) - r_P - L \leq 0 ~.
\]
Convex duality also clarifies the relation between the mixture coefficients $\lambda_k$ and the calibration constants~ $r_{P_k}$. Increasing the weight of a distribution in the mixture is equivalent to reducing the corresponding calibration coefficient. This observation then leads to a plethora of saddle-point seeking algorithms such as Uzawa iterations~\citep{uzawa-58} (see Section A.4.).

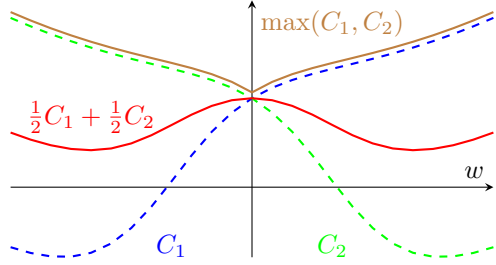
\begin{figure}[t]
    \centering
    \vspace{1ex}
    \begin{tikzpicture}
        \begin{axis}[width=8cm,height=5cm,
          axis lines=middle, 
          xmin = -3, xmax = 3, ymin = -0.6, ymax = 1.6, 
          xlabel={$w$}, ticks=none]
            \addplot[thick,blue,domain=-3:3,dashed] {
                tanh(1 + x) + (x^2 / 20)};
            \addplot[thick,green,domain=-3:3,dashed] {
                tanh(1 - x) + (x^2 / 20)};
            \addplot[thick,brown,domain=-3:3] {
                0.05 + max(tanh(1 - x) + (x^2 / 20),
                           tanh(1 + x) + (x^2 / 20) ) };  
            \addplot[thick,red,domain=-3:3] {
                (tanh(1 + x) + tanh(1 - x)) / 2 + (x^2 / 20) };  
            \node[blue] at (axis cs:-1,-0.5) {$C_1$};
            \node[green] at (axis cs:1,-0.5) {$C_2$};
            \node[brown] at (axis cs:1,1.4){
                $\max(C_1,C_2)$};
            \node[red] at (axis cs:-2,0.6){
                $\tfrac12{C_1}+\tfrac12{C_2}$};
        \end{axis}
    \end{tikzpicture}
    \caption{ \label{fig:counterexample}
    The minimum $w^*=0$ of $\max\{C_1(w),C_2(w)\}$ is a stationary point of the mixture cost $C_{\mathrm{mix}}(w) = \tfrac12C_1(w)+\tfrac12C_2(2)$. However this stationary point is not a local minimum but a local maximum of the mixture cost.}
\end{figure}

The nonconvex case is more challenging because the stationary points identified by Theorem~\ref{th:one} need not be local minima.
Consider for instance the two real functions
\[
   C_1(w) = \tanh(1+w) + \epsilon w^2
   ~~\text{and}~~
   C_2(w) = \tanh(1-w) + \epsilon w^2
\]
where the term $\epsilon w^2$ with $0<\epsilon\ll1$
is only present to ensure that each of these functions
has a well defined optimum. As shown in Figure~\ref{fig:counterexample}, their
maximum $\max\{C_1(w),C_2(w)\}$ has a a 
minimum in $w^*=0$. As predicted by Theorem~\ref{th:one}, 
this solution is a stationary point of the the mixture $C_{\mathrm{mix}}=\tfrac12C_1(w)+\tfrac12C_2(w)$.
However this stationary point is not a local minimum
but a local \emph{maximum}.

This situation is in fact easy to understand. The solution $w^*{=}0$ of the problem $\min_w\max\{C_1(w),C_2(w)\}$ (Figure~\ref{fig:counterexample}) falls in negative curvature regions of the functions $C_1$ and $C_2$. As a result any mixture of these two costs also has negative curvature in $w^*$. Therefore, the stationary point $w^*$ cannot be a local minimum. It is also easy to see that this situation cannot occur when the optimum of the DRO problem is achieved in points where the individual cost functions have positive curvature. Since all mixtures must also have positive curvature on these points, the stationary points can only be local minima. This remark is important because learning algorithms for deep learning problems tend to follow trajectories where the Hessian is very flat apart from a few positive eigenvalues~\citep{sagun-2018}. Weak negative curvature directions always exist ---even when the algorithm stops making progress--- but are very weak. Although we lack a good understanding of these landscapes, it seems a safe bet to assume that the situation presented in Figure~\ref{fig:counterexample} is often cured by overparametrization. The following section reaches a similar conclusion with closer look at a popular family of DRO algorithms.

\subsection{Lagrangian algorithms for DRO}
\label{sec:lagrangianalgo}
 
The calibrated DRO problem \eqref{eq:drocal} is easily rewritten as a constrained optimization problem by introducing a slack variable $L$:
\[
    \min_{w,L}  L  \quad \text{s.t.} \quad
        \forall P\in\Q.~~~~ C_P(w) - r_P - L \leq 0 ~.
\]
With convex loss function, finite $\Q$, and under adequate qualification conditions \citep{boyd-vandenberghe-2004,bertsekas-2009}, convex duality theory suggests to write the Lagrangian
\[
   L(w,M,\lambda_1\dots\lambda_K) = M + \sum_k \lambda_k 
     \big( C_{P_k}(w) - r_P - M \big) ~,
\]
and solve instead the dual problem,
\[
   \max_{\lambda_k\geq0} ~ \left\{
     D(\lambda_1\dots\lambda_K) \stackrel{\Delta}{=}
       \min_{w,M} L(w,M,\lambda_1\dots\lambda_K)  \right\}
\]
The solution of this problem must satisfy $\sum_k\lambda_k=1$ because the dual $D(\lambda_1\dots\lambda_k)$ is $-\infty$ when this is not the case. With this knowledge, the dual problem
becomes
\[
   \max_{\begin{array}{c}
          \scriptstyle\lambda_k\geq0\\[-0.8ex]
          \scriptstyle\sum_k\!\lambda_k=1\end{array}} 
    \left\{
    D(\lambda_1\dots\lambda_K) = 
      \left( \min_w  \sum_k\lambda_k C_{P_k}(w) \right) - \left( \sum_k \lambda_k r_{P_k}
      \right) \right\}~.
\]
The inner optimization problem is precisely the minimization of the expected risk with respect to the mixture $\sum_k\lambda_k P_k$ and therefore lends itself to many popular gradient descent methods. The mixture coefficients $\lambda_k$ must then be slowly adjusted by ascending the outer optimization objective \citep{uzawa-58}.

\IncMargin{1em} 
\setcounter{algocf}{0}
\begin{algorithm}[t]
\label{alg:lagrangianalgo}
\DontPrintSemicolon
\LinesNumbered
\SetKw{KwInput}{Input:}
\SetKw{KwOutput}{Output:}
\SetKwFunction{Descend}{Descend}
\SetKwFunction{Cost}{Cost}
\SetKwFunction{Acceptable}{Acceptable}
\SetKwFunction{ArgMax}{ArgMax}
\KwInput{Equally sized training sets $D_k$ for $k=1\dots K$}\;
\KwInput{Calibration coefficients $r_k$. 
         Initial weights $w_0$.}\;
\KwInput{Temperature $\beta$. 
         Stopping threshold $\epsilon$.}\;
\KwOutput{A sequence of weights $w_t$.}\;
\BlankLine
$t \leftarrow K$\;
$\lambda_k \leftarrow 1/t ~~\forall k$\;
\Repeat {$\max_k|\lambda_k-\delta_k| < t\epsilon$}{
    $w_{t+1} \leftarrow 
        \Descend\big(w_t, \{ D_1\star\lambda_1 \dots D_K\star\lambda_K \} \big)$ \; 
    $c_k \leftarrow 
        \Cost\big(w_{t+1}, \{ D_k \}\big)~~\forall k$ \;  
    $\delta_k \leftarrow 
        \frac{1}{Z} \exp(\beta (c_k-r_k)) ~~ \forall k$ 
        \hfill\emph{--- with $Z$ such that $\sum_k\delta_k=1$} \;
    $\lambda_k \leftarrow \frac{1}{t+1}
        (t\lambda_k + \delta_k)  ~~ \forall k$ \;
    $t \leftarrow t+1$\;
}
\Return $w_t$
\caption{\label{alg:descent}
  A typical Lagrangian DRO algorithm.}
\end{algorithm}

Algorithm~\ref{alg:lagrangianalgo} is a typical example of this strategy. Although this particular instance uses a temperature parameter $\beta$ to smooth the mixture coefficient update rule, it is also common to focus on a single term with $\beta=+\infty$. When this is the case, each outer iteration of Algorithm~\ref{alg:lagrangianalgo} merely amounts to augmenting the training set with an extra copy of the examples associated with most adverse subpopulation.

Because of their simplicity and effectiveness, such Lagrangian DRO algorithms are  widely used with deep learning system with nonconvex objectives~\citep{sagawa-2020,augustin2020adversarial}. The theoretical results discussed in this appendix provide a measure support for this practice.

A crucial assumption for this algorithm is the idea that increasing the weight of a distribution in the mixture amounts to finding a local DRO minimum with a lower calibration coefficient for that distribution. This is true in the convex case. This requires a more precise discussion in the nonconvex case. Suppose for instance that one modifies the mixture coefficients by slightly increasing $\lambda_1$ by a small $\delta>0$ and re-normalizing:
\[
    \lambda'_1 = \tfrac1Z (\lambda_1 + \delta) \qquad
    \lambda'_k = \tfrac1Z \lambda_k ~~ \forall k>1
\]
Such a change can yield two outcomes. Either $w^*$ remains a local minimum of the new expected cost mixture, or we can follow a descent trajectory and reach a new local minimum~$w'$:
\begin{equation}
\label{eq:nwnc}
    Z \sum_k \lambda'_k C_{P_k}(w')
      ~<~ Z \sum_k \lambda'_k C_{P_k}(w^*)~.
\end{equation}
\begin{itemize}
    \item[$i$)] Let us first assume that the old cost function increases when one moves from its local minimum $w^*$ to the  local minimum $w'$ of the new cost function
    \begin{equation}
    \label{eq:nwoc}
      \sum_k \lambda_k C_{P_k}(w')
      ~\geq~ \sum_k \lambda_k C_{P_k}(w^*)
    \end{equation}
    Subtracting \eqref{eq:nwoc} from \eqref{eq:nwnc} yields
    \[
    \delta C_{P_1}(w') < \delta C_{P_1}(w^*)~,  
    \]
    which, according to Theorem~\ref{th:converse}, means that the new local minimum~$w'$ is a local minimum of a DRO problem with a reduced calibration coefficient for distribution $P_1$, just as for convex losses.

    \item[$ii$)] However, it is also conceivable that \eqref{eq:nwoc} does not hold. This means that the new minimum $w'$ achieves a lower cost than $w^*$ for both the old and new mixture costs. In other words, tweaking the mixture allowed us to escape the attraction basin of the local minimum $w^*$. From the perspective of algorithm~\ref{alg:lagrangianalgo}, this disrupts the determination of the mixture coefficient, but this is nevertheless progress because both the old and new mixture costs are lower. In theory, this can only happen a finite number of times in a neural network because there is only a finite number of attraction basins. In practice, this never happens: stochastic gradient descent in neural networks usually follows a path with slowly decreasing cost without hopping from one attraction basin to another one \citep{goodfellow2014qualitatively, sagun-2018}. 
\end{itemize}
  
As mentioned earlier, it is also conceivable that $w^*$ remains a local minimum with the new mixture cost. Algorithm~\ref{alg:lagrangianalgo} then keeps increasing the weight of distribution $P_1$ as longs as the cost $C_{P_1}(w^*)=C_{P_1}(w')$ remains too high with respect to the desired calibration coefficients. This last case covers two distinct scenarios. 
\begin{itemize}
    \item[$iii$)] The Lagrangian algorithm could keep increasing the weight of the first distribution without moving away from the local minimum $w^*$. The inner loop eventually minimizes the empirical risk for the first distribution only, yet without achieving progress. This suggests that we have reached a disappointing bound on the best performance achievable with our model using training data sampled from this first distribution.
    \item[$iv)$] 
    Alternatively, the old mixture local minimum $w^*$ could stop being a local minimum of the new mixture once the first distribution weight reaches a certain threshold.  Consider for instance the problem of Figure~\ref{fig:counterexample}. Even though the DRO minimum corresponds to a local \emph{maximum} of the mixture cost $C_{\mathrm{mix}}(w)=\tfrac12C_1+\tfrac12C_2$, Theorem~\ref{th:converse} tells us that both minima of  this mixture cost are also local DRO minima, albeit for different calibration constants $r_i$. Figure~\ref{fig:recalibrate} shows the case where $r_2>r_1$. Figure~\ref{fig:recalibrateandjump} shows that increasing the weight of the second cost function beyond a certain threshold eventually erases the left minimum and causes Algorithm~\ref{alg:lagrangianalgo} to jump to the condition $r_1>r_2$. In other words, our algorithm is not able to simultaneously keep both cost functions as low as they could separately be.  This either suggests that these two goals are incompatible, or that the model does not have enough capacity to simultaneously achieve them together. As usual with neural networks, the remedy is  overparametrization\dots
\end{itemize}

\begin{figure}[t]
    \centering
    \vspace{1ex}
    \begin{tikzpicture}
        \begin{axis}[width=8cm,height=5cm,
          axis lines=middle, 
          xmin = -4, xmax = 2, ymin = -0.6, ymax = 1.2, 
          xlabel={$w$}, ticks=none]
            \addplot[thick,blue,domain=-4:2,dashed] {
                tanh(1 + x) + (x^2 / 20) + 0.55};
            \addplot[thick,green,domain=-4:2,dashed] {
                tanh(1 - x) + (x^2 / 20) - 1.2};
            \addplot[thick,brown,domain=-4:2] {
                0.05 + max(tanh(1 + x) + (x^2 / 20) + 0.55,
                           tanh(1 - x) + (x^2 / 20) - 1.2 ) };  
            \addplot[thick,red,domain=-4:2] {
                (tanh(1 - x) + tanh(1 + x)) / 2 + (x^2 / 20) };  
            \node[blue] at (axis cs:-3,-0.3) {$C_1-r_1$};
            \node[green] at (axis cs:-1,-0.5) {$C_2-r_2$};
            \node[brown] at (axis cs:-2.4,1) {
                $\max(C_1-r_1,C_2-r_2)$};
            \node[red] at (axis cs:1,0.9) {
                $\tfrac12{C_1}+\tfrac12{C_2}$};
        \end{axis}
    \end{tikzpicture}
    \caption{\label{fig:recalibrate} Both minima of $C_{\mathrm{mix}}(w)=\tfrac12C_1+\tfrac12C_2$ are solutions of a DRO problem, albeit one with different calibration constants $r_1$ and $r_2$. Here $r_2>r_1$.}
    \centering
    \vspace{1ex}
    \begin{tikzpicture}
        \begin{axis}[width=8cm,height=5cm,
          axis lines=middle, clip=false,
          xmin = -3, xmax = 3, ymin = -0.6, ymax = 1.6, 
           ticks=none]
            \addplot[thick,red,domain=-3:3.1] {
                (tanh(1 + x) + tanh(1 - x))/2 + (x^2 / 20) };  
            \addplot[thick,red!60!black,domain=-3:3.1] {
                (tanh(1 + x) + 3*tanh(1 - x))/4 + (x^2 / 20) };  
            \addplot[thick,red!40!black,domain=-3:3.1] {
                (tanh(1 + x) + 7*tanh(1 - x))/8 + (x^2 / 20) };  
            \node[red] at (axis cs:4,0.5){
                $\tfrac12{C_1}+\tfrac12{C_2}$};
            \node[red!60!black] at (axis cs:4,0){
                $\tfrac14{C_1}+\tfrac34{C_2}$};
            \node[red!40!black] at (axis cs:4,-0.33){
                $\tfrac18{C_1}+\tfrac78{C_2}$};  
        \end{axis}
    \end{tikzpicture}
    \caption{\label{fig:recalibrateandjump} Increasing the weight of the second distribution beyond a certain threshold erases the first minimum and causes Algorithm~\ref{alg:lagrangianalgo} to jump to the other minimum which is a calibrated DRO minimum for $r_1>r_2$.
    }
\end{figure}

One can derive two conclusions from this brief analysis. First, as long as we use a Lagrangian descent algorithm to solve the DRO problem, there is little point being concerned about stationary points of the mixture cost that are not local minima because (1) the algorithm is not going to find them anyway, and (2) overparametrizing the network is likely to make them disappear anyway (scenario $iv$ above). Second, the most concerning scenario is the case where a single distribution or subpopulation dominates the DRO problem because our model is unable to achieve a satisfactory performance even when it is trained to minimize the
expected cost for that distribution only. When this is the case, DRO cannot help.

\end{document}